\PassOptionsToPackage{sort}{natbib}
\documentclass{article}
\pdfoutput=1 
\usepackage{preprint,times}

\usepackage{amsmath,amsfonts,bm}

\def\eqref#1{equation~\ref{#1}}

\def\1{\bm{1}}

\DeclareMathAlphabet{\mathsfit}{\encodingdefault}{\sfdefault}{m}{sl}
\SetMathAlphabet{\mathsfit}{bold}{\encodingdefault}{\sfdefault}{bx}{n}

\newcommand{\E}{\mathbb{E}}

\newcommand{\R}{\mathbb{R}}

\usepackage{graphicx} 
\usepackage[english]{babel}
\usepackage[nopatch=footnote]{microtype}      

\usepackage{amssymb,dsfont,amsmath,amsthm,mathtools,mathrsfs}
\usepackage{mathtools}
\usepackage{graphicx}
\usepackage{tikz-cd}
\usepackage{pb-diagram}
\usepackage{comment}
\usepackage{tcolorbox}
\usepackage{enumitem}
\setlist[enumerate]{leftmargin=.5in}
\setlist[itemize]{leftmargin=.5in}
\makeatletter
\newcommand{\myitem}[1]{%
	\item[#1]\protected@edef\@currentlabel{#1}%
}
\makeatother

\theoremstyle{plain}
\newtheorem{theorem}{Theorem}[section]
\newtheorem{assumption}[theorem]{Assumption}
\newtheorem{lemma}[theorem]{Lemma}
\newtheorem{proposition}[theorem]{Proposition}
\newtheorem{corollary}[theorem]{Corollary}

\newtheoremstyle{my_def}
  {\topsep}   
  {\topsep}   
  {\normalfont}  
  {0pt}       
  {\bfseries\itshape} 
  {.}         
  {5pt plus 1pt minus 1pt} 
  {}          
\theoremstyle{my_def}
\newtheorem{definition}[theorem]{Definition}
\newtheorem{example}[theorem]{Example}

\newtheoremstyle{my_rem}
  {1.0\topsep}   
  {1.0\topsep}   
  {\normalfont}  
  {0pt}       
  {\bfseries\itshape} 
  {.}         
  {5pt plus 1pt minus 1pt} 
  {}          
\theoremstyle{my_rem}
\newtheorem{remark}[theorem]{Remark}

\numberwithin{equation}{section}
\numberwithin{theorem}{section}

\newcommand{\RR}{\mathbb{R}}
\newcommand{\Pp}{\mathcal{P}}

\newcommand{\cP}{\mathcal{P}}

\renewcommand{\d}[1]{\mathrm{d}}

\renewcommand*\d{\mathop{}\!\mathrm{d}}

\newcommand{\slot}{{\,\cdot\,}}
\newcommand{\push}{_{\#}}

\newcommand{\sfW}{\mathsf{W}}
\newcommand{\sfT}{\mathsf{T}}
\newcommand{\N}{\mathbb{N}}
\newcommand{\ep}{\varepsilon}

\newcommand{\defeq}{\coloneqq}

\DeclarePairedDelimiterX{\iptemp}[2]{\langle}{\rangle}{#1, #2}

\DeclarePairedDelimiterX{\normtemp}[1]{\lVert}{\rVert}{#1}
\newcommand{\norm}{\normtemp}
\DeclarePairedDelimiterX{\abstemp}[1]{\lvert}{\rvert}{#1}
\newcommand{\abs}{\abstemp}
\DeclarePairedDelimiterX{\trtemp}[1]{(}{)}{#1}

\DeclarePairedDelimiterX{\SEtemp}[2]{(}{)}{#1, #2}

\DeclarePairedDelimiterX{\SGtemp}[1]{(}{)}{#1}

\def\qfa{\quad\text{for all}\quad}

\def\qas{\quad\text{as}\quad}
\def\qa{\quad\text{and}\quad}
\def\qf{\quad\text{for}\quad}
\def\qw{\quad\text{where}\quad}
\def\qin{\quad\text{in}\quad}
\def\qon{\quad\text{on}\quad}

\def\qif{\quad\text{if}\quad}

\renewcommand{\eqref}[1]{\textup{(\ref{#1})}}
\usepackage{hyperref}
\usepackage{url}
\newcommand{\compactemail}[1]{{\fontfamily{cmtt}\fontseries{m}\selectfont #1}}

\newcommand{\mytitle}{
        Universal Approximation of Measure-to-Measure Operators by Pushforwards
}

\title{\mytitle}

\makeatletter
\let\inserttitle\@title
\makeatother

\author{Takashi Furuya\\
Doshisha University, RIKEN AIP\\
\compactemail{tfuruya@mail.doshisha.ac.jp} \\
\And
Nicholas H. Nelsen \\
UT Austin \\
\compactemail{nnelsen@oden.utexas.edu} \\
\And
Frank Cole \\
UCLA \\
\compactemail{fcole99@g.ucla.edu}
}

\iclrfinalcopy

\ifpdf
\hypersetup{
	pdftitle={Universal Approximation of Measure-to-Measure Operators by Pushforwards},
	pdfauthor={T. Furuya, N. H. Nelsen, and F. Cole}
}
\fi

\begin{document}

\maketitle

\begin{abstract}
Many learning tasks map an input distribution to an output distribution. A natural way to model such an operator is to transform each input sample using a continuous function that may depend on the entire input distribution, and then take the distribution of the transformed samples. This defines a measure-dependent pushforward model and includes measure-theoretic formulations of transformers. We ask when such models can approximate arbitrary continuous operators between spaces of probability measures. We first show that universal approximation fails when atomic inputs are allowed: some continuous measure-to-measure operators that split or redistribute atomic mass cannot be approximated arbitrarily well by deterministic pushforward models. We then introduce the uniform level set condition, which requires a continuous measure-dependent scalarization whose shrinking level set neighborhoods carry uniformly vanishing mass over the input family. This condition is satisfied, in particular, by compact families of absolutely continuous measures. On every compact family satisfying this condition, we prove that any continuous measure-to-measure operator with outputs of finite $p$-th moment can be uniformly approximated, in the $p$-Wasserstein distance, by continuous measure-dependent pushforwards. Combining our theorem with existing approximation results for measure-dependent in-context maps yields universal approximation by measure-theoretic transformers. We also extend the framework to continuously-varying source measures, yielding a corresponding universality result for a class of pushforward models that are closely aligned with cross-attention architectures.
\end{abstract}

\section{Introduction}
Many scientific and learning problems involving distributions can naturally be formulated as maps between spaces of probability measures. 
Given an input distribution, one would like to produce another distribution that depends continuously on the input. 
Examples include the prior-to-posterior map in Bayesian inference \citep{stuart2010inverse,sprungk2020local}, proximal maps on Wasserstein spaces \citep{ambrosio2008gradient}, and solution maps of McKean--Vlasov dynamics \citep{kolokoltsov2010nonlinear}.

A particularly important class of measure-to-measure maps is given by \emph{pushforward models}. 
In such models, each sample from the input distribution is transformed by a map that may itself depend on the entire input distribution, and the collection of transformed samples determines the output distribution.
This structure appears naturally in measure-theoretic formulations of transformers and attention mechanisms
\citep{castin2025unified,furuya2024transformers,geshkovski2025mathematical,rigollet2026mean,sander2022sinkformers,vuckovic2020mathematical}.
Closely related pushforward constructions also arise in generative modeling, including flow matching \citep{lipman2022flow} and conditional normalizing flows \citep{winkler2019learning}.
This leads to the following fundamental approximation question:
\begin{tcolorbox}
\centering
    \textbf{To what extent can an arbitrary continuous operator between spaces of probability measures be approximated by continuous pushforward models whose transformation is allowed to depend on the input measure?}
\end{tcolorbox}

In this paper, we first show that universal approximation by continuous pushforward models fails in general when atomic input measures are allowed.
We provide counterexamples induced by continuous measure-to-measure operators that split or redistribute atomic mass, which cannot in general be approximated by deterministic pushforward models.

We then identify a sufficient geometric-measure-theoretic condition on a compact family of input measures, which we call the \emph{uniform level set condition}. 
Roughly speaking, this condition requires the existence of a continuous measure-dependent scalarization whose shrinking level set neighborhoods carry uniformly vanishing mass across the entire family of input measures. 
Under this condition, we prove that every continuous measure-to-measure operator with finite $p$-th moment outputs can be uniformly approximated in the $p$-Wasserstein distance by continuous pushforward models.

The uniform level set condition is substantially more flexible than absolute continuity with respect to  the Lebesgue measure. In particular, it is automatically satisfied by compact families of non-atomic measures in one dimension, as well as by compact families dominated by a common non-atomic reference measure in higher dimensions.
It also applies to certain families of singular measures supported on continuously-varying lower-dimensional structures, such as normalized arclength measures on a compact family of injective curves.
At the same time, it necessarily excludes atomic measures, which is consistent with the obstruction described previously.

Our results have direct implications for the expressive power of transformers.
Measure-theoretic transformers act on an input distribution through a measure-dependent transformation of its tokens and therefore induce a pushforward map at the level of probability measures.
Combining our main approximation theorem with existing universality results for in-context maps, we establish universal approximation of continuous measure-to-measure operators by measure-theoretic transformers on compact families satisfying the uniform level set condition.
Last, we extend the pushforward approximation framework to continuously-varying source measures, which relates to cross-attention.

\paragraph{Related work.}
The universal approximation properties of transformers have been extensively studied in the classical sequence-to-sequence setting \citep{vaswani2017attention}. Early results established universality for transformers and their sparse variants
\citep{yun2019transformers,yun2020n}, followed by extensions concerning constrained architectures, prompting, approximation rates, and more refined descriptions of the expressive power of attention
\citep{kratsios2021universal,petrov2024prompting,liu2026attention,chen2025fundamental,cheng2026unified,jiao2025approximation,jiang2024approximation,takakura2023approximation,wang2024understanding,shen2026understanding,havrilla2024understanding,calvello2025continuum}.

A complementary line of work studies transformers in a measure-theoretic formulation, where a collection of tokens is represented by its empirical measure and attention is extended to probability measures. 
This viewpoint has led to mathematical descriptions of attention and transformer dynamics
\citep{vuckovic2020mathematical,sander2022sinkformers,geshkovski2025mathematical,castin2025unified,rigollet2026mean,bach2025learning,bach2026learning,burger2025analysis},
as well as approximation results for functions whose inputs involve probability measures
\citep{furuya2024transformers,furuya2026approximation,biswal2024universal,geshkovski2024measure,fraiman2026expressive,cole2026context} and characterizations of the smoothness of attention layers acting on probability measures \citep{castin2024smooth, vuckovic2021regularity,kim2021lipschitz}.

More closely related to the present work are approximation and exact representation results in which both the input and output are probability measures. In work of \citet{furuya2025transformers}, transformer-induced measure-to-measure maps are studied through the lens of support-preserving transformations. 
Related exact transport map representations of continuous measure-to-measure transformations are investigated by \citet{lavenant2026continuous}. 
These works naturally lead to structural conditions under which the output does not split the mass carried by individual input particles. 
Such conditions are particularly well suited to empirical measures, since an empirical input is then transformed into another empirical measure with compatible particle structure.

The question studied in the present paper is different.
We ask under what conditions on a family of input measures can an arbitrary continuous measure-to-measure operator be approximated by deterministic pushforward models \emph{without imposing a support-preserving or non-splitting condition on the target map}.
This distinction is \emph{essential}: natural maps such as Bayesian updating or joint measure conditioning may continuously change the weights of atoms and hence fall outside exact deterministic pushforward representations.
Our results characterize a regime in which this obstruction disappears at the level of uniform approximation and thereby connect general continuous measure-to-measure operators with transformer-induced pushforward models.

\paragraph{Contributions.}
The main contributions of this work are summarized as follows.

\begin{enumerate}[label={(C\arabic*)},labelindent=0.05em,leftmargin=*]
\item \textbf{An obstruction caused by atomic measures.}
We show that continuous measure-dependent pushforward models are not universal on general compact subsets of probability measures (Proposition~\ref{prop:obstructions}).
In particular, deterministic pushforwards cannot split atomic mass or freely modify atomic weights. This yields explicit continuous measure-to-measure operators that remain separated from every deterministic pushforward model by a strictly positive $1$-Wasserstein distance (Examples~\ref{ex:diff-card} and ~\ref{ex:diff-weight}).

\item \textbf{Universal approximation under a uniform level set condition.}
We introduce the uniform level set condition for compact families of input measures (Assumption~\ref{ass:uniform-level-set-condition}).
We prove that, under this condition, every continuous measure-to-measure operator with outputs of finite $p$-th moment can be uniformly approximated in the $p$-Wasserstein distance by continuous measure-dependent pushforward maps (Theorem~\ref{thm:main-multi}). 
The condition necessarily excludes atomic measures, but is satisfied by compact families of non-atomic measures in one dimension, compact families dominated by a common non-atomic reference measure, and certain families of singular measures supported on continuously-varying lower-dimensional structures (Examples~\ref{ex:one-dimensional-uniform-levelset}, \ref{ex:dominated-family-uniform-levelset}, and \ref{ex:curve-measure-uniform}).

\item \textbf{Consequences for transformer architectures.}
By combining our main approximation theorem with existing universality results for measure-theoretic in-context maps, we establish universal approximation of continuous measure-to-measure operators by measure-theoretic transformers on compact families satisfying the uniform level set condition (Theorem~\ref{thm:transformer-universality}).
We further extend the pushforward approximation framework to continuously-varying source measures, which has applications to cross-attention (Corollary~\ref{cor:main-cross}).
\end{enumerate}

\paragraph{Notation.}
Let $|\slot|$ denote the Euclidean norm. 
We denote by 
$\cP(\mathbb{R}^d)$ the set of all Borel probability measures on $\mathbb{R}^d$ and, 
for $1\le p<\infty$,
by $
\mathcal P_p(\mathbb{R}^d)
\defeq 
\{
\mu\in\mathcal P(\mathbb{R}^d)\colon
\int_{\R^d} |x|^p\,\mu(dx)<\infty
\}$ the set of probability measures with finite $p$-th moment.
The subset of \emph{non-atomic} or \emph{atomless} probability measures is
\begin{align}\label{eqn:atomless}
\mathcal P_{\rm na}(\mathbb{R}^d)
\defeq 
\left\{
\mu\in\mathcal P(\mathbb{R}^d)\colon 
\mu(\{x\})=0
\ \text{ for every }x\in\mathbb{R}^d
\right\}.
\end{align}
We equip $\mathcal P(\mathbb{R}^d)$ with the topology of weak convergence of probability measures.
Whenever $\mathcal P_p(\mathbb{R}^d)$ is considered, it is equipped with the $p$-Wasserstein metric $\sfW_p$ given by
\[
\sfW_p(\mu,\nu)
\defeq
\left(
\inf_{\pi\in\Pi(\mu,\nu)}
\int_{\mathbb{R}^d\times\mathbb{R}^d}|x-y|^p\,\pi(dx,dy)
\right)^{1/p},
\]
where $\Pi(\mu,\nu)$ is the set of couplings of $\mu$ and $\nu$.
For a measurable map $T\colon\mathbb{R}^d\to\mathbb R^{d'}$ and $\mu\in\cP(\mathbb{R}^d)$, the pushforward measure $T\push\mu\in\cP(\R^{d'})$ is defined by
$
T_\#\mu(A)
\defeq 
\mu(T^{-1}(A))$
for all
$A\subset\R^{d'}$ Borel.
For topological spaces $X$ and $Y$, we denote by
$\mathcal{C}(X,Y)$ the space of continuous maps from $X$ to $Y$; the topology on spaces of probability measures is always understood as above.

\section{The pushforward approximation problem}
\label{sec:pushforward-approx}

This section formulates the approximation problem studied in this work, illustrates it with two examples of continuous measure-to-measure operators, and identifies atomic obstructions.

\subsection{The measure-to-measure approximation problem}\label{sec:pushforward-approx-problem}

In what follows, let $p\in[1,\infty)$,  $d\in\N$, and $d'\in\N$.
We investigate whether for every compact set $K \subset \Pp_p(\mathbb{R}^d)$ with respect to $\sfW_p$, every $F \in \mathcal{C}(K,\Pp_p(\RR^{d'}))$, and every $\varepsilon>0$, there exists a continuous map $G \in \mathcal{C}(K\times \mathbb{R}^d,\RR^{d'})$ such that
\[
\sup_{\mu\in K}
\sfW_p\bigl(G(\mu,\slot)_\#\mu, F(\mu)\bigr)
<\varepsilon.
\]
In other words, we ask whether continuous measure-to-measure operators can be uniformly approximated by continuous measure-dependent pushforward models. Pushforward models are especially attractive from a computational perspective \citep{marzouk2017sampling}. This approximation problem arises naturally in several settings.
We illustrate it with an example.

\begin{example}[Prior-to-posterior map in Bayesian inference]
In Bayesian inverse problems, updating a prior probability measure using observed data modeled by a likelihood naturally defines a map from prior to posterior measures \citep{stuart2010inverse,nelsen2026operator}.
Let $\Phi\colon\mathbb{R}^d\to\mathbb{R}$ denote the negative log-likelihood associated with
the observed data. 
For a prior measure $\mu\in\mathcal{P}_p(\mathbb{R}^d)$, the corresponding posterior measure is defined by
\[
F_\Phi(\mu)(du)
\defeq 
\frac{\exp\bigl(-\Phi(u)\bigr)}{Z_\mu}\,\mu(du),
\qw
Z_\mu
\defeq 
\int_{\mathbb{R}^d} \exp\bigl(-\Phi(v)\bigr)\,\mu(dv).
\]
Thus, Bayesian inference defines the prior-to-posterior map
$
F_\Phi\colon
\mathcal{P}_p(\mathbb{R}^d)
\to
\mathcal{P}_p(\mathbb{R}^d)
$.
If $\Phi$ is continuous and bounded from below, then $F_\Phi$ is well-defined and $\sfW_p$-continuous \citep[Lemma~16]{sprungk2020local}.
Notice that $F_\Phi$ generally changes the weights of the input measure while preserving its null sets. 
In particular, for an atomic prior $\mu=\sum_{i=1}^N a_i\delta_{x_i}$, the posterior is
\[
F_\Phi(\mu)
=
\sum_{i=1}^N
\frac{a_i e^{-\Phi(x_i)}}
{\sum_{j=1}^N a_j e^{-\Phi(x_j)}}
\delta_{x_i}.
\]
Hence, although the support is unchanged, the weights are in general different.
Consequently, $F_\Phi(\mu)$ cannot in general be represented as a deterministic pushforward  of the form $G(\mu, \slot)_\#\mu$. 
This provides a natural example of a continuous measure-to-measure operator that lies beyond exact deterministic pushforward representations on atomic and nonatomic measures, while still falling within the approximation problem considered in this work.
\end{example}

See Appendix~\ref{app:additional-pushforward-approx} for a further example based on proximal maps between Wasserstein spaces.

\subsection{Obstructions caused by atomic measures}

In general, the answer to the approximation question posed in Section~\ref{sec:pushforward-approx-problem} is negative when atomic input measures are allowed. More precisely, the following holds.

\begin{proposition}
\label{prop:obstructions}
There exist a compact set
$
K\subset \mathcal P_p(\mathbb R^d)
$ in $\sfW_p$, a continuous map
$
F\in \mathcal C(K,\mathcal P_p(\mathbb R^{d'})),
$
and $\varepsilon>0$ such that there is no continuous map $G\colon K\times \mathbb R^d\to\mathbb R^{d'}$ satisfying
\[
\sup_{\mu\in K}
\sfW_p\bigl(G(\mu,\slot)\push\mu, F(\mu)\bigr)
<\varepsilon.
\]
\end{proposition}

The obstruction already appears in the simplest possible examples.

\begin{example}[Different support cardinalities]
\label{ex:diff-card}
Let $K\defeq \{\delta_0\}\subset\mathcal P(\mathbb R)$ and define
\[
F(\delta_0)\defeq 
\frac12\delta_{-1}+\frac12\delta_1.
\]
Since $K$ is a singleton, $F\colon K\to\mathcal P_1(\mathbb R)$ is continuous.
For any continuous $G\colon K\times \mathbb R\to\mathbb R$, we have
$
G(\delta_0,\slot)\push\delta_0
=
\delta_{G(\delta_0,0)}
$.
Therefore,
\begin{align*}
\sfW_1\bigl(
G(\delta_0,\slot)\push\delta_0,
F(\delta_0)
\bigr)
=
\frac12\left|G(\delta_0,0)+1\right|
+
\frac12\left|G(\delta_0,0)-1\right|
\ge 1
\end{align*}
by the triangle inequality. Hence
\[
\inf_{G\in\mathcal C(K\times \mathbb R, \mathbb R)}
\sup_{\mu\in K}
\sfW_1\bigl(
G(\mu, \slot)\push\mu, F(\mu)
\bigr)
\ge 1.
\]
Thus, a deterministic pushforward cannot approximate a map that splits a single input atom into two.
\end{example}

The preceding example changes the number of atoms. 
However, this is not the essential issue: an obstruction remains even when the input and output measures have the same support cardinality.

\begin{example}[Same support cardinality, different weights]
\label{ex:diff-weight}

Let $K\defeq \left\{\frac12\delta_{-1}+\frac12\delta_1\right\}$ and define
\[
\nu = \frac12\delta_{-1}+\frac12\delta_1 \qa F(\nu)\defeq \frac13\delta_{-1}+\frac23\delta_1.
\]
Again, $F\colon K\to\mathcal P_1(\mathbb R)$ is continuous.
For any continuous map $G\colon K\times \mathbb R\to\mathbb R$, it holds that
\[
G(\nu,\slot)\push\nu
=
\frac12\delta_{G(\nu,-1)}
+
\frac12\delta_{G(\nu, 1)}.
\]
Proposition~\ref{prop:two_atom_counterexample} delivers the lower bound
\[
\sfW_1\bigl(G(\nu,\slot)\push\nu, F(\nu) \bigr)
\ge \frac13.
\]
Consequently,
\[
\inf_{G\in\mathcal C(K\times \mathbb R,\mathbb R)}
\sup_{\mu\in K}
\sfW_1\bigl(
G(\mu, \slot)\push\mu, F(\mu)
\bigr)
\ge \frac13.
\]
Thus, even when the input and output distributions have the same number of atoms, a deterministic pushforward cannot in general modify their weights.
\end{example}

\section{Universal approximation by continuous pushforward models}
\label{sec:universal-approx}
This section establishes our main universal approximation result for continuous measure-dependent pushforward models.

\subsection{Uniform level set condition}
We introduce a sufficient condition on the family $K$ of input measures under which universal approximation holds.

\begin{assumption}
\label{ass:uniform-level-set-condition}
Let $K \subset \mathcal{P}_p(\mathbb{R}^d)$ be $\sfW_p$-compact. 
There exists $\varphi\in \mathcal{C}(K\times \mathbb{R}^d ,\mathbb R)$ such that
\begin{equation}
\label{eq:uniform-level-set-condition}
\lim_{ r\downarrow0}
\sup_{\mu\in K}\sup_{t\in\mathbb R}
\mu\bigl(\{x\in\mathbb{R}^d\colon |\varphi(\mu,x)-t|\le  r\}\bigr)
=0.
\end{equation}
\end{assumption}

We call \eqref{eq:uniform-level-set-condition} the \emph{uniform level set condition}. Intuitively, Assumption~\ref{ass:uniform-level-set-condition} requires the existence of a continuous scalarization $\varphi$ whose level sets carry uniformly vanishing mass over the family $K$. 
In particular, the condition is satisfied when all measures in $K$ are absolutely continuous with respect to Lebesgue measure: one may take, for example, any linear projection $\varphi(\mu,x)\defeq x\cdot v$ in nonzero direction $v$ independently of $\mu$.
More generally, as shown below, absolute continuity with respect to Lebesgue
measure is far from necessary. The importance of Assumption \ref{ass:uniform-level-set-condition} is that it enables us to represent the one-dimensional uniform distribution $\mathrm{Unif}[0,1]$ exactly by a continuous pushforward model. This is a key ingredient in the proof of Theorem \ref{thm:main-multi}, our main approximation result.

The uniform level set condition is closely related to the atomless property \eqref{eqn:atomless}.
\begin{remark}[Uniform level set condition implies non-atomic family]\label{rem-unif-level-set}
If $K$ is such that Assumption~\ref{ass:uniform-level-set-condition} holds, then $K\subset \mathcal P_{\rm na}(\mathbb{R}^d)$. Indeed, fix $\mu\in K$ and $x_0\in\mathbb{R}^d$, and set
$t=\varphi(\mu,x_0)$. Then
\[
\{x_0\}
\subset
\{x\in\mathbb{R}^d\colon 
|\varphi(\mu, x)-\varphi(\mu, x_0)|\le  r\}
\]
for any $r\geq 0$.
Hence
\[
\mu\bigl(\{x_0\}\bigr)
\le
\sup_{\nu\in K}\sup_{t\in\mathbb R}
\nu\bigl(
\{x\in\mathbb{R}^d\colon |\varphi(\nu,x)-t|\le r\}
\bigr)
\to 0
\]
as $ r\downarrow0$ by \eqref{eq:uniform-level-set-condition}. Since $x_0$ was arbitrary, $\mu$ is non-atomic as asserted.
\end{remark}
Moreover, the uniform level set condition \eqref{eq:uniform-level-set-condition} corresponding to $K$ and $\varphi$ is equivalent to the family $\{\varphi(\mu,\slot)\push\mu\}_{\mu\in K}\subset\mathcal{P}_{\rm na}(\R)$ of scalar pushforwards being non-atomic; see Remark~\ref{rem:uniform-level-set-condition-equiv}.

To better understand Assumption~\ref{ass:uniform-level-set-condition}, we present several examples of compact families that satisfy it.
\begin{example}[One-dimensional case]
\label{ex:one-dimensional-uniform-levelset}
Let $K\subset \mathcal P_p(\mathbb R)$ be compact with respect to $\sfW_p$. 
Then the identity function $\varphi\colon \mathbb R\to\mathbb R$ given by $\varphi(x)\defeq x$ satisfies
\[
\lim_{ r\downarrow0}
\sup_{\mu\in K}\sup_{t\in\mathbb R}
\mu\bigl(
\{x\in\mathbb R\colon |\varphi(x)-t|\le r\}
\bigr)
=0.
\]
In particular, every compact family of non-atomic probability measures on $\mathbb R$ satisfies the uniform level set condition \eqref{eq:uniform-level-set-condition}. See Appendix~\ref{app:one-dimensional-uniform-levelset} for the proof. 
\end{example}

Remark \ref{rem-unif-level-set} shows that if a compact subset $K \subset \mathcal{P}_p(\R^d)$ satisfies Assumption \ref{ass:uniform-level-set-condition}, then every element of $K$ is atomless. Example \ref{ex:one-dimensional-uniform-levelset} establishes the converse when $d=1$; this gives a complete characterization of subsets satisfying Assumption \ref{ass:uniform-level-set-condition} in the one-dimensional setting.

\begin{example}[Families dominated by a common non-atomic measure]
\label{ex:dominated-family-uniform-levelset}
Let $\rho\in\mathcal P_{\rm na}(\mathbb R^d)$ and let
\[
K
\subset
\mathcal{P}_{\mathrm{ac},p}(\rho) \defeq  \{\mu\in\mathcal P_p(\mathbb R^d)\colon \mu\ll\rho\}
\]
be compact with respect to $\sfW_p$.
Then there exists a continuous function $\varphi\colon \mathbb R^d\to\mathbb R$ such that
\[
\lim_{ r\downarrow0}
\sup_{\mu\in K}\sup_{t\in\mathbb R}
\mu\bigl(
\{x\in\mathbb R^d\colon |\varphi(x)-t|\le r\}
\bigr)
=0.
\]
\end{example}

See Appendix~\ref{app:dominated-family-uniform-levelset} for the proof. By taking $\rho$ to be the Lebesgue measure, Example \ref{ex:dominated-family-uniform-levelset} implies that Assumption \ref{ass:uniform-level-set-condition} is satisfied by any $\sfW_p$-compact family of densities. Importantly, the only requirement in Example \ref{ex:dominated-family-uniform-levelset} is that $\rho$ be non-atomic. In particular, Example \ref{ex:dominated-family-uniform-levelset} applies to compact families of probability measures that are supported on a low-dimensional manifold $\mathcal{M}$ and admit a density with respect to some fixed atomless reference measure on $\mathcal{M}$, e.g., a uniform measure on $\mathcal{M}$. Such families of measures arise naturally under the manifold hypothesis in machine learning.

\begin{example}[Measures supported on a compact family of curves]
\label{ex:curve-measure-uniform}

Let $Q\subset\mathbb R^d$ be compact and let $K'\subset \mathcal{C}^1([0,1],Q)$ be compact. 
Assume that there exists $c>0$ such that
\[
|\dot\gamma(s)|\ge c
\qf \gamma\in K'\qa s\in[0,1],
\]
and that every $\gamma\in K'$ is injective.
For each $\gamma\in K'$, define the normalized arclength measure
\[
\sigma_\gamma
\defeq 
\left(\int_0^1|\dot\gamma(s)|\,ds\right)^{-1}
\gamma_\#\bigl(|\dot\gamma(s)|\,ds\bigr).
\]
Then there exists a continuous function $\widetilde\varphi\colon  K'\times Q \to[0,1]$ such that $\widetilde\varphi(\gamma, \gamma(s))=s$ for $\gamma\in K'$ and $s\in[0,1]$,
and
\[
\lim_{r\downarrow0}
\sup_{\gamma\in K'}
\sup_{t\in\mathbb R}
\sigma_\gamma
\bigl(\{
x\in Q\colon 
|\widetilde\varphi(\gamma,x)-t|\le r
\}
\bigr)
=0.
\]
Assume, in addition, that the curves in $K'$ are uniquely determined by their images, that is,
\[
\gamma_1([0,1])=\gamma_2([0,1])
\quad\text{implies}\quad
\gamma_1=\gamma_2.
\]
Then the family
$
K\defeq \{\sigma_\gamma\colon \gamma\in K'\} \subset \mathcal{P}_p(Q)
$
satisfies the uniform level set condition \eqref{eq:uniform-level-set-condition}.
\end{example}

See Appendix~\ref{app:curve-measure-uniform} for the proof. Example~\ref{ex:curve-measure-uniform} illustrates that the uniform level set condition can accommodate singular measures whose supports themselves vary continuously. 
Such families naturally arise when data distributions are concentrated near evolving or parameter-dependent low-dimensional structures, rather than on a single fixed reference manifold.

\subsection{Main result}\label{sec-main-results}

We are now ready to state our main universal approximation theorem. 
The key idea is to use the uniform level set condition to find a continuous in-context map that exactly pushes forward each input measure to a common uniform distribution on $[0,1]$.

\begin{theorem}  
\label{thm:main-multi}
Let $K\subset\mathcal P_p(\mathbb{R}^d)$ satisfy Assumption~\ref{ass:uniform-level-set-condition}.
Then for any $F \in \mathcal{C}(K,\mathcal{P}_p(\mathbb R^{d'}))$ and any $\varepsilon>0$, there exists a continuous map $G\colon  K\times \mathbb{R}^d \to \mathbb R^{d'}$ such that
\begin{equation*}
\sup_{\mu\in K}
\sfW_p\bigl( G(\mu,\slot)_\#\mu,
F(\mu)\bigr) < \varepsilon.
\end{equation*}
\end{theorem}

\begin{proof}[Sketch of proof]
The full proof is given in Appendix~\ref{app:proof-main-theorem}.
Here we describe the construction together with the estimates used below.
For $R>0$, let
\[
F_R(\mu)\defeq (\Pi_R)_\#F(\mu),
\]
where $\Pi_R$ is the metric projection onto the closed ball $\overline{B_R(0)}$ of radius $R$. Define
\[
E_{{\rm tail},p}(R)
\defeq
\sup_{\nu\in F(K)}
\left(
\int_{\{|y|>R\}}|y|^p\,\nu(dy)
\right)^{1/p}.
\]
By Lemma~\ref{lem:uniform-truncation}, $E_{{\rm tail},p}(R)\to0$ as $R\to\infty$ and
\[
\sup_{\mu\in K}
\sfW_p\bigl(F(\mu),F_R(\mu)\bigr)
\le E_{{\rm tail},p}(R).
\]
Next, for $h>0$, choose an $h$-net $\{y_1,\ldots,y_N\}\subset\overline{B_R(0)}$ and a continuous partition of unity subordinate to the corresponding cover.
The covering number may be chosen so that
\[
N=N(R,h)
\le
C_{d'}
\left(1+\frac{R}{h}\right)^{d'},
\]
where $C_{d'}>0$ depends only on $d'$.
This yields continuous weights $\alpha_i\colon K\to[0,1]$ such that
\[
\sup_{\mu\in K}
\sfW_p\bigl(F_{R,h}(\mu),F_{R}(\mu)\bigr)
\le h, \qw F_{R,h}(\mu)
\defeq
\sum_{i=1}^N\alpha_i(\mu)\delta_{y_i}.
\]
Writing $s_i(\mu)\defeq \sum_{j=1}^i\alpha_j(\mu)$ and $ s_0(\mu)\defeq 0$,
the discontinuous interval map
\[
u\mapsto Q_0(\mu,u)
\defeq
\sum_{i=1}^N
\mathbf 1_{(s_{i-1}(\mu),s_i(\mu)]}(u)\,y_i
\]
pushes $\mathrm{Unif}[0,1]$ exactly onto $F_{R,h}(\mu)$.
We smooth the transition points of $Q_0$ on intervals of width $\delta>0$ to obtain a bounded continuous map $Q_\delta$ whose $\sfW_p$ contribution to errors in the transition regions is bounded by $2R\bigl(2N(R,h)\delta\bigr)^{1/p}$. Under Assumption~\ref{ass:uniform-level-set-condition}, the construction in Lemma~\ref{lem:Uniform-approximate-uniformizer} delivers the existence of $H_0\in \mathcal{C}(K\times \mathbb{R}^d, [0,1])$ such that
\[
H_0(\mu,\slot)_\#\mu = \mathrm{Unif}[0,1] \qfa \mu\in K.
\]
Combining this exact continuous uniformizer with the preceding arguments and defining
\[
G(\mu,x)
\defeq
Q_\delta\bigl(\mu, H_0(\mu,x)\bigr),
\]
the triangle inequality gives the asserted approximation after successively choosing $R$, $h$, and $\delta$ appropriately as functions of $\varepsilon$.
\end{proof}

For a fixed $\mu\in\mathcal{P}_{\rm na}(\R^d)$, it is known that the set of deterministic Monge couplings $(\mathrm{id}, g)\push\mu$ induced by continuous transport maps $g\colon\R^d\to\R^{d'}$ with source measure $\mu$ is dense in the set of all joint distributions in $\mathcal{P}(\R^d\times\R^{d'})$ with first marginal $\mu$ \cite[Prop.~2.2]{beiglbock2018denseness}. Theorem~\ref{thm:main-multi} can then be interpreted as the operator analog of this ``diagonal approximation'' result in which $\mu$ is no longer fixed, but is instead allowed to vary uniformly over the compact set $K$.

\begin{remark}[Quantitative error bound]
\label{rem:quantitative-main}
The preceding proof actually yields the upper bound
\[
\sup_{\mu\in K}
\sfW_p\bigl(
G(\mu,\slot)_\#\mu,
F(\mu)
\bigr)
\le
E_{{\rm tail},p}(R)
+
h
+
2R\left[
2C_{d'}
\left(1+\frac{R}{h}\right)^{d'}\delta
\right]^{1/p}.
\]
Moreover, the construction gives the uniform bound
$
\|G\|_{\mathcal{C}(K\times \mathbb{R}^d,\R^{d'})}\le R
$.
\end{remark}
\section{Applications to transformer architectures}
\label{sec:application}

We now apply our main approximation theorem to measure-theoretic transformer architectures.

\subsection{Measure-theoretic transformers}
We first recall the measure-theoretic formulation of transformers.

\begin{definition}[Measure-theoretic in-context attention maps]
For $H\in\N$, a \emph{measure-theoretic in-context multi-head attention map} $\Gamma_\theta \colon \mathcal P(\mathbb R^d) \times \mathbb{R}^d \to \mathbb R^d$ is defined by
\begin{equation}\label{eq-mf-self-attn}
\Gamma_\theta(\mu,x)
\defeq
x
+
\sum_{h=1}^H
W^h
\int_{\R^d}
\frac{
\exp\bigl(
\langle Q^h x, K^h y \rangle
\bigr)
}{
\int_{\R^d}
\exp\bigl(
\langle Q^h x, K^h z \rangle
\bigr)
\, \mu(dz)
}
\, V^h y \, \mu(dy),
\end{equation}
where $\theta\defeq (W^h,Q^h,K^h,V^h)_{h=1}^{H}$, $(\mu,x)\in
\mathcal{P}(\mathbb{R}^d)\times\mathbb{R}^d$, and $W^h$, $Q^h$, $K^h$, and $V^h$ are the head, query, key, and value matrices, respectively. See Remark~\ref{rmk-softmax} for a note about finiteness of \eqref{eq-mf-self-attn}.
\end{definition}

In \eqref{eq-mf-self-attn}, $\Gamma_\theta$ acts on a distinguished token $x$ while depending on the input probability measure $\mu$, and can therefore be viewed as an in-context map. 
The standard discrete attention mechanism is recovered when $\mu$ is an empirical measure $\mu=\frac{1}{n}\sum_{i=1}^n\delta_{x_i}$ associated with tokens $x_1,\ldots,x_n$. In this case, the integral in the definition of $\Gamma_\theta$ reduces to the usual finite sum over the tokens.
The measure-theoretic formulation removes the need to fix the number of tokens in advance and provides a unified description of attention for arbitrary probability measures, including non-atomic measures.
For further details on measure-theoretic formulations of attention and in-context maps, see, e.g., the work of \citet{castin2024smooth,furuya2025transformers,furuya2024transformers,geshkovski2025mathematical,geshkovski2023emergence,geshkovski2024measure,castin2025unified}.

\begin{definition}[Composition of measure-theoretic in-context maps]\label{defn:composition_incontext}
Let
\[
\Gamma_1\colon 
\mathcal P(\mathbb R^{d_0})\times\mathbb R^{d_0}
\to
\mathbb R^{d_1} \quad
\text{ and } \quad
\Gamma_2\colon 
\mathcal P(\mathbb R^{d_1})\times\mathbb R^{d_1}
\to
\mathbb R^{d_2}.
\]
Their \emph{composition} is defined by
\[
\bigl(\Gamma_2\diamond\Gamma_1\bigr)(\mu,x)
\defeq 
\Gamma_2\bigl(\Gamma_1(\mu,\slot)_\#\mu,\Gamma_1(\mu,x) \bigr).
\]
\end{definition}

\begin{definition}[Measure-theoretic transformer]
\label{def:transformer}
Let
$
\Gamma_{\theta_1},\ldots,\Gamma_{\theta_L}
$
be measure-theoretic in-context attention maps as in \eqref{eq-mf-self-attn} and
$
\phi_{1},\ldots,\phi_{L}
$
be context-free multilayer perceptron (MLP) maps.
The associated deep transformer in-context map $\sfT\colon 
\mathcal P(\mathbb R^d)\times\mathbb R^d
\to
\mathbb R^{d'}$ is defined by
\begin{equation}
\label{eq:in-cotext-map}
\sfT
\defeq 
\phi_L\diamond\Gamma_{\theta_L}
\diamond\cdots\diamond
\phi_1 \diamond\Gamma_{\theta_1}.
\end{equation}
The corresponding \emph{measure-theoretic transformer} is the measure-to-measure operator defined by
\begin{equation*}
\mathcal P(\mathbb R^d) \ni \mu
\mapsto  
\sfT(\mu,\slot)\push \mu \in \mathcal P(\mathbb R^{d'}).
\end{equation*}
\end{definition}

\Citet{furuya2024transformers} show that transformer in-context maps of the form \eqref{eq:in-cotext-map} are universal approximators for general continuous in-context mappings defined on compact token domains; the crucial difficulty is that the present paper works on the whole of $\R^d$, not a compact domain. Nevertheless, combining modifications of this universality result (Proposition~\ref{prop:transformer-universality-incontext-unbounded}) with Theorem~\ref{thm:main-multi} yields the following universal approximation result for measure-to-measure operators.

\begin{theorem}
\label{thm:transformer-universality}
Let $K \subset \mathcal{P}_p(\mathbb R^d)$ be a $\sfW_p$-compact set satisfying Assumption~\ref{ass:uniform-level-set-condition}.
Then for any $F\in \mathcal C(K,\mathcal P_p(\mathbb R^{d'}))$ and any $\varepsilon>0$, there exists a deep transformer in-context map $\sfT\colon \mathcal P(\mathbb R^d)\times\mathbb R^d \to \mathbb R^{d'}$ of the form \eqref{eq:in-cotext-map} such that
\[
\sup_{\mu\in K}
\sfW_p\bigl(
\sfT(\mu,\slot)_\#\mu,
F(\mu)
\bigr)
<\varepsilon.
\]
\end{theorem}

See Appendix~\ref{app:transformer-universality} for the proof. Theorem \ref{thm:transformer-universality} proves that the measure-to-measure operators defined by transformers are universal approximators of continuous measure-to-measure operators defined on a compact set $K \subset \mathcal{P}_p(\R^d),$ provided that $K$ satisfies the uniform level set condition \eqref{eq:uniform-level-set-condition}. In many practical scenarios, the transformer $\mathsf{T}$ is not applied to $\mu\in K$ directly---which is atomless by Remark~\ref{rem-unif-level-set}---but instead to an atomic measure $\mu_N$ supported on $N$ samples from $\mu$. The error analysis in this case can still be handled by first controlling the $\sfW_p$ error between $\mathsf{T}(\mu_N,\slot)_{\#}\mu_N$ and $\mathsf{T}(\mu,\slot)_{\#}\mu$ uniformly over $K$ \citep{cole2026stability} and then applying Theorem~\ref{thm:transformer-universality} to bound the $\sfW_p$ error between $\mathsf{T}(\mu,\slot)_{\#}\mu$ and $F(\mu)$.
This approach requires Assumption~\ref{ass:uniform-level-set-condition} and discretizes $\mu\in K$ with samples. A different approach removes Assumption~\ref{ass:uniform-level-set-condition} on the $\sfW_p$-compact set $K$ altogether by convolving the (now possibly atomic) input measure $\mu\in K$ with a standard isotropic Gaussian distribution that has a small variance $\sigma^2$ to obtain a distribution $\mu_\sigma$; see Corollary~\ref{cor:noise_approx} in Appendix~\ref{app:additional-application} for details. Given samples $X_1, \ldots, X_N$ from $\mu$, one can combine the two approaches by approximating $\mu_{\sigma}$ by an empirical measure of the form $\widehat{\mu}_{\sigma} = \frac{1}{N} \sum_{i=1}^{N} \delta_{X_i + \sigma Z_i},$ where $Z_1, \ldots, Z_N$ are independent Gaussian samples with mean zero and identity covariance matrix.

\begin{remark}[Domain of measure-theoretic attention]\label{rmk-softmax}
    To be precise, the measure-theoretic attention map $\Gamma_{\theta}$ in \eqref{eq-mf-self-attn} is only well-defined for inputs $(\mu,x)$ such that 
    $$
    \int_{\R^d} \exp \bigl( \langle Q^h x, K^h y \rangle \bigr) \, \mu(dy) < \infty\qa \int_{\R^d} \exp \bigl( \langle Q^h x, K^h y\rangle \bigr) V^h y\, \mu(dy) < \infty
    $$
    for all $h\in \{1,\ldots, H\}$. These conditions always hold when $\mu$ has compact support; \citet{cole2026stability} give more general conditions on $\mu$ based on sub-Gaussianity. Although Theorem~\ref{thm:transformer-universality} applies to input measures $\mu\in K$ with possibly full support on $\R^d$, the specific construction only requires evaluating self-attention layers at compactly supported measures. This is because the transformer constructed in the proof of Theorem~\ref{thm:transformer-universality} utilizes an explicit initial MLP layer which truncates the input measure to a common compact set.
\end{remark}

\subsection{Continuously-varying source measures and connection to cross-attention}
The preceding result is an approximation theorem for measure-theoretic transformers based on self-attention, in which the source measure of the transport is the input measure itself. The next corollary generalizes the main pushforward approximation result, Theorem~\ref{thm:main-multi} from Section~\ref{sec-main-results}, to the setting in which the
source measure of the pushforward is allowed to depend continuously on the input
measure, but need not equal the input measure. This is particularly relevant for analyzing cross-attention
architectures, where the measure carrying the query tokens may differ
from the measure representing the context and may even live on a
different ambient space.

\begin{corollary}
\label{cor:main-cross}
Let $K\subset \mathcal P_p(\mathbb{R}^d)$ be $\sfW_p$-compact. Let $\eta\colon
K\to \mathcal P_p(\mathbb R^m)$ be $\sfW_p$-continuous. Assume that there exists
$\varphi\in
\mathcal C(K\times \mathbb R^m,\mathbb R)$
such that
\begin{equation}
\label{eq:cross-uniform-level-set}
\lim_{r\downarrow0}
\sup_{\mu\in K}\sup_{t\in\mathbb R}
\,
\eta(\mu)
\bigl(\{
x\in\mathbb R^m\colon
|\varphi(\mu,x)-t|\le r
\}
\bigr)
=0.
\end{equation}
Then for every $F\in
\mathcal C(
K,\mathcal P_p(\mathbb R^{d'}))$ and every $\varepsilon>0$, there exists $G\in\mathcal{C}(K\times \R^m,\mathbb R^{d'})$ such that
\[
\sup_{\mu\in K}
\sfW_p\bigl(
G(\mu,\slot)_\#\eta(\mu),
F(\mu)
\bigr)
<
\varepsilon.
\]
\end{corollary}

Typical source measures covered by Corollary~\ref{cor:main-cross} include a fixed non-atomic measure, $\eta(\mu)\equiv\rho$ (e.g., an isotropic Gaussian), and a smoothed input measure, $\eta(\mu)=\rho\ast\mu$, where $\rho$ is absolutely continuous.
The proof of Corollary~\ref{cor:main-cross} follows the same construction as that of
Theorem~\ref{thm:main-multi} and may be found in Appendix~\ref{app:main-cross-proof}.

\begin{remark}[Relation to cross-attention]
\label{rem:cross-attention-architecture}
Cross-attention allows one collection of query tokens to attend to a distinct collection of context tokens and is a central component of encoder--decoder transformers \citep{vaswani2017attention} and related architectures such as the Perceiver~\citep{jaegle2021perceiver}.
Corollary~\ref{cor:main-cross} is directly connected to architectures in which the source tokens attend to the context tokens through cross-attention and are otherwise processed token-wise, without self-attention among the source tokens.
Combined with universality for measure-dependent in-context maps---see, e.g., Theorem~1 due to \citet{furuya2024transformers} or Proposition~\ref{prop:transformer-universality-incontext-unbounded}---it yields an analog of Theorem~\ref{thm:transformer-universality} for this particular cross-attention configuration. 
However, this argument does not by itself establish universality for more general architectures that combine both self-attention among source tokens with cross-attention among context tokens.
\end{remark}

\section{Conclusion}
In this paper, we studied the approximation of continuous measure-to-measure operators by continuous \emph{pushforward models}, a subset of mappings which are naturally motivated by transformers. We showed that atomic measures present an obstruction to universal approximation due to the inability of pushforward models to split or adjust the weights of atoms. To remove this obstruction, we identified a sufficient condition for universal approximation on a compact set $K \subset \mathcal{P}_p(\mathbb{R}^d),$ which we termed the uniform level set condition. Notably, the uniform level set condition is weaker than absolute continuity and allows for measures whose supports admit low-dimensional structure. Our main result also implies a uniform approximation theorem for measure-theoretic transformers.

This work is the first to study universal approximation properties of pushforward models under general assumptions on the target operator; however, several important questions still remain. First, while the uniform level set condition is sufficient for universal approximation, it may not be necessary. An interesting mathematical question is to identify the minimal assumptions required on a compact set $K \subset \mathcal{P}_p(\mathbb{R}^d)$ to enable universal approximation of continuous operators on $K$ by pushforward models.
Second, certain measure-to-measure operators of interest, such as the joint-to-conditional operator in Bayesian inference, may not be continuous \citep[Appendix~E]{tsimpos2026one}; approximation of such maps by pushforward models is thus beyond the scope of this paper. Last, an important problem is obtaining quantitative approximation results for measure-theoretic transformers, where the number of parameters is controlled by the smoothness of the target operator and the covering number of the domain of approximation. We leave these important directions to future work.

\clearpage

\subsection*{AI use statement}

In this work, we used generative AI tools to assist in the writing of proofs, formulating mathematical claims and refining them, providing critical ingredients for proving mathematical claims, and assist with translation.
We also used generative AI tools for checking proof arguments, improving the organization, presentation, readability, and structure of the paper, language editing, brainstorming, and literature search assistance.
We have not used generative AI tools to help develop theoretical models or conceptual frameworks, propose or refine hypotheses, design or provide feedback on research methodology or experiments, or interpret results. Moreover, we did not use generative AI tools to generate synthetic data sets, implement methods, clean and reformat datasets, or support qualitative and thematic data analysis; these tasks are not applicable to this theoretical work.

All mathematical statements and proofs suggested or modified with the assistance of generative AI were independently examined and verified by the authors. 
Relevant references and attribution were checked against the original sources, and the final mathematical arguments and wording were reviewed and revised by the authors. 
We take responsibility for the final content of this work, including text, claims, and other artifacts produced with the aid of generative AI.

\subsubsection*{Acknowledgments}
T.F. acknowledges funding from JSPS KAKENHI (grants JP24K16949
and 25H01453), JST CREST (JPMJCR24Q5), and JST ASPIRE (JPMJAP2329).
The research of N.H.N. is supported by a Klarman Fellowship through Cornell University's College of Arts \& Sciences and by startup funds at The University of Texas at Austin.
This project arose out of the 2nd Workshop on Machine Learning in Infinite Dimensions, which was sponsored by the ETH Z\"urich and the ProbAI Hub on the Mathematical and Computational Foundations of AI.

\bibliographystyle{preprint}
\bibliography{ref}


\clearpage

\appendix

\makeatletter
\@addtoreset{equation}{section}
\@addtoreset{theorem}{section}
\makeatother

\renewcommand{\theequation}{\thesection.\arabic{equation}}
\renewcommand{\thetheorem}{\thesection.\arabic{theorem}}

\setcounter{equation}{0}
\setcounter{theorem}{0}

\begin{center}
	{\bf Supplementary Material for:} \\ 
	\inserttitle \\
\end{center}

\section{Auxiliary results}
\label{app:additional}
This appendix contains auxiliary results and additional supporting material.

\subsection{Additional material for Section~\ref{sec:pushforward-approx}}\label{app:additional-pushforward-approx}
We now provide another example of a continuous measure-to-measure operator; this one is based on proximal maps. Proximal maps and minimizing-movement schemes can be defined on general metric spaces; see, e.g., the work of \citet{ambrosio2008gradient}. Here we consider such a construction on the $1$-Wasserstein space.
\begin{example}[$\sfW_1$-proximal map]
Let $\Omega\subset\mathbb{R}^d$ be compact with positive Lebesgue measure. Define the entropy functional
\[
\operatorname{Ent}(\nu)
\defeq 
\begin{cases}
\displaystyle
\int_\Omega \rho(x)\log \rho(x)\,dx,
& \text{if density}\  
\rho\defeq \frac{d\nu}{dx}
\ \text{exists},\\[2mm]
+\infty,
& \text{otherwise}.
\end{cases}
\]
For $\tau>0$ and $\mu\in\mathcal{P}_1(\Omega)$, consider the variational problem
\[
F_\tau(\mu)
\defeq 
\operatorname*{argmin}_{\nu\in\mathcal{P}_1(\Omega)}
\left\{
\operatorname{Ent}(\nu)
+
\frac{1}{2\tau}\sfW_1^2(\mu,\nu)
\right\}.
\]
Since $\Omega$ is compact, $\mathcal{P}_1(\Omega)=\mathcal{P}(\Omega)$ and this space is compact with respect to $\sfW_1$. 
Moreover, $\operatorname{Ent}$ is lower semicontinuous and strictly convex, while
\[
\nu\mapsto \sfW_1^2(\mu,\nu)
\]
is convex. 
Hence, the above minimization problem admits a unique minimizer.
Therefore, $F_\tau$ defines a measure-to-measure operator
\begin{align*}
F_\tau\colon
\mathcal{P}_1(\Omega)
&\to
\mathcal{P}_1(\Omega),\\
\mu&\mapsto F_\tau(\mu).
\end{align*}
We next observe that $F_\tau$ is continuous with respect to $\sfW_1$.
Indeed, let $\mu_n\to\mu$ in $\sfW_1$ and set
\[
\nu_n\defeq F_\tau(\mu_n).
\]
By compactness of $\mathcal{P}_1(\Omega)$, every subsequence of $(\nu_n)_{n\in\mathbb{N}}$ has a further subsequence, still denoted by $(\nu_n)_{n\in\mathbb{N}}$, such that
\[
\nu_n\to\nu_*
\qin \sfW_1
\]
for some $\nu_*\in\mathcal{P}_1(\Omega)$.
For every $\eta\in\mathcal{P}_1(\Omega)$, the minimizing property gives
\[
\operatorname{Ent}(\nu_n)
+
\frac{1}{2\tau}\sfW_1^2(\mu_n,\nu_n)
\le
\operatorname{Ent}(\eta)
+
\frac{1}{2\tau}\sfW_1^2(\mu_n,\eta).
\]
Using the lower semicontinuity of $\operatorname{Ent}$ and the continuity of $\sfW_1$, we obtain
\[
\operatorname{Ent}(\nu_*)
+
\frac{1}{2\tau}\sfW_1^2(\mu,\nu_*)
\le
\operatorname{Ent}(\eta)
+
\frac{1}{2\tau}\sfW_1^2(\mu,\eta).
\]
Since $\eta$ is arbitrary, $\nu_*$ is a minimizer of the variational problem associated with $\mu$. 
By uniqueness,
\[
\nu_*=F_\tau(\mu).
\]
Hence every convergent subsequence of $(\nu_n)_{n\in\mathbb{N}}$ has the same limit $F_\tau(\mu)$, and therefore
\[
F_\tau(\mu_n)\to F_\tau(\mu)
\qin \sfW_1.
\]
Thus,
\[
F_\tau\in
\mathcal C\bigl(
\mathcal{P}_1(\Omega),
\mathcal{P}_1(\Omega)
\bigr).
\]
This map also illustrates a limitation of exact deterministic pushforward representations on atomic measures.
Suppose that $0\in\Omega$ and consider the input
\[
\mu\defeq \delta_0.
\]
The minimum value in the definition of $F_\tau(\delta_0)$ is finite, since there exist absolutely continuous probability measures on $\Omega$ with finite entropy. Consequently,
\[
\operatorname{Ent}\bigl(F_\tau(\delta_0)\bigr)<\infty,
\]
and hence
\[
F_\tau(\delta_0)\ll dx.
\]
In particular, $F_\tau(\delta_0)$ is non-atomic. 
On the other hand, for
every deterministic map
\[
G\colon \mathcal{P}_1(\Omega)\times\Omega\to\Omega,
\]
we have
\[
G(\delta_0,\slot)_\#\delta_0
=
\delta_{G(\delta_0,0)},
\]
which is a Dirac measure. Therefore,
\[
F_\tau(\delta_0)
\neq
G(\delta_0,\slot)_\#\delta_0
\]
for every such $G$. Thus, $\sfW_1$-proximal maps provide natural examples of continuous
measure-to-measure operators which, on atomic inputs, cannot in general be
represented exactly by deterministic pushforward models.
\end{example}

\medskip

To conclude this appendix, we provide the proof of the lower bound asserted in Example~\ref{ex:diff-weight}.
\begin{proposition}
\label{prop:two_atom_counterexample}
Let
\[
\nu_{u,v}\defeq \frac12(\delta_u+\delta_v)\qa 
\nu^\ast\defeq \frac13\delta_{-1}+\frac23\delta_1,
\]
where $u\le v$. 
Then
\[
\inf_{u\le v}
\sfW_1(\nu_{u,v},\nu^\ast)
=
\frac13.
\]
The minimum is attained at $(u,v)=(-1,1)$.
\end{proposition}

\begin{proof}
In one dimension, the $1$-Wasserstein distance admits the quantile representation
\[
\sfW_1(\mu,\nu)
=
\int_0^1
|Q_\mu(t)-Q_\nu(t)|\,dt,
\]
where $Q_\mu$ and $Q_\nu$ denote the corresponding quantile functions.
For
\[
\nu_{u,v}=\frac12(\delta_u+\delta_v) \qa  u\le v,
\]
we have
\[
Q_{\nu_{u,v}}(t)
=
\begin{cases}
u, & 0<t<\frac12,\\[1mm]
v, & \frac12\le t<1,
\end{cases}
\]
whereas
\[
Q_{\nu^\ast}(t)
=
\begin{cases}
-1, & 0<t<\frac13,\\[1mm]
1, & \frac13\le t<1.
\end{cases}
\]
Hence
\begin{align*}
\sfW_1(\nu_{u,v},\nu^\ast)
&=
\int_0^{1/3}|u+1|\,dt
+
\int_{1/3}^{1/2}|u-1|\,dt
+
\int_{1/2}^{1}|v-1|\,dt
\\
&=
\frac13|u+1|
+
\frac16|u-1|
+
\frac12|v-1|.
\end{align*}
We now minimize this expression under the constraint $u\le v$. 
If $u\le 1$, the choice $v=1$ is admissible and minimizes the last term.
Thus it remains to minimize
\[
f(u)\defeq 
\frac13|u+1|
+
\frac16|u-1| \qf u\le1.
\]
For $u\le-1$,
\[
f(u)
=
-\frac12u-\frac16,
\]
which is decreasing as $u$ increases and therefore attains its minimum at $u=-1$.
For $-1\le u\le1$,
\[
f(u)
=
\frac12+\frac16u,
\]
which is increasing and hence again attains its minimum at $u=-1$.
Therefore,
\[
\inf_{u\le1} f(u)=f(-1)=\frac13.
\]
If $u>1$, then the constraint $v\ge u$ implies
\[
|v-1|=v-1\ge u-1.
\]
Consequently,
\begin{align*}
\sfW_1(\nu_{u,v},\nu^\ast)
&\ge
\frac13(u+1)
+
\frac16(u-1)
+
\frac12(u-1)
\\
&=
u-\frac13
>
\frac13.
\end{align*}
Thus the global minimum is $\frac13$, attained at
\[
(u,v)=(-1,1)
\]
as claimed.
\end{proof}

\subsection{Additional material for Section~\ref{sec:universal-approx}}\label{app:universal-approx}
As discussed in Section~\ref{sec:universal-approx}, the uniform level set condition \eqref{eq:uniform-level-set-condition} for a compact set $K$ and continuous scalarization $\varphi\in\mathcal{C}(K\times\R^d,\R)$ is equivalent to the condition that $\varphi(\mu,\slot)\push\mu$ is non-atomic for every $\mu\in K$. We now elaborate on the proof of this fact.
\begin{remark}[Equivalence of uniform level set condition and pointwise non-atomic family]
\label{rem:uniform-level-set-condition-equiv}
Let $K \subset \mathcal{P}_p(\mathbb{R}^d)$ be $\sfW_p$-compact. Let $\varphi\in \mathcal{C}(K\times \mathbb{R}^d ,\mathbb R)$. The following are equivalent:
\begin{align}
&\lim_{ r\downarrow0}
\sup_{\mu\in K}\sup_{t\in\mathbb R}
\mu\bigl(\{x\in\mathbb{R}^d\colon |\varphi(\mu,x)-t|\le  r\}\bigr)
=0\qa\label{eqn:equiv1}\\
&\mu\bigl(\{x\in\mathbb{R}^d\colon \varphi(\mu,x)=t\}\bigr)=0\qfa \mu\in K\quad\text{and \ all}\quad  t\in\R.\label{eqn:equiv2}
\end{align}
Indeed, \eqref{eqn:equiv1} implies \eqref{eqn:equiv2} by the argument leading to \eqref{eqn:pw_atomless} in the proof of Theorem~\ref{thm:main-multi}. To see the other direction, first define $\nu_\mu\defeq\varphi(\mu,\slot)\push\mu$. Let \eqref{eqn:equiv2} hold, which is equivalent to
\begin{align*}
    \nu_\mu\in\mathcal{P}_{\rm na}(\R)\qfa\mu\in K.
\end{align*}
Define
\[
L\defeq \{\nu_\mu\colon \mu\in K\}
\subset \mathcal{P}_{\rm na}(\R).
\]
We claim that the map
\[
\mu\mapsto\nu_\mu
\]
is continuous from $K\subset (\mathcal{P}_p(\R^d),\sfW_p)$ into $\mathcal{P}(\R)$ with the weak topology. To this end, let $\mu_n\to \mu$ in $\sfW_p$ and hence in distribution. Let $f$ be any continuous and bounded function on $\R$. Then
\begin{align*}
    \int_{\R^d}f\bigl(\varphi(\mu,x)\bigr)\,\mu_n(dx)\to \int_{\R^d}f\bigl(\varphi(\mu,z)\bigr)\,\mu(dz)\qas n\to\infty
\end{align*}
by the weak convergence of $\mu_n\to\mu$ because $f\circ\varphi(\mu,\slot)$ is bounded and continuous. On the other hand, for any $R>0$, it holds that
\begin{align*}
    \abs[\bigg]{\int_{\R^d}\Bigl[f\bigl(\varphi(\mu_n,x)\bigr) - f\bigl(\varphi(\mu,x)\bigr)\Bigr]\,\mu_n(dx)}&\leq \sup_{x\in \overline{B_R(0)}}\abs[\big]{f\bigl(\varphi(\mu_n,x)\bigr) - f\bigl(\varphi(\mu,x)\bigr)}\\
    &\qquad\qquad + 2\norm{f}_{\mathcal{C}(\R)} \sup_{k\in\N}\mu_k\bigl(\R^d\setminus B_R(0)\bigr)
\end{align*}
by splitting the integral. For fixed $R$, the first term on the right-hand side of the preceding display tends to zero as $n\to\infty$ by continuity of $f\circ\varphi$ and compactness of closed $\R^d$ balls. Then sending $R\to\infty$, the second term converges to zero by uniform tightness of the weakly convergent sequence $\{\mu_n\}_{n\in\N}$. An application of the triangle inequality then shows that
\begin{align*}
    \int_{\R}f(t)\,\nu_{\mu_n}(dt)=\int_{\R^d}f\bigl(\varphi(\mu_n,x)\bigr)\,\mu_n(dx)\to \int_{\R^d}f\bigl(\varphi(\mu,z)\bigr)\,\mu(dz)=\int_{\R}f(s)\,\nu_{\mu}(ds)
\end{align*}
as $n\to\infty$, so $\nu_{\mu_n}\to\nu_\mu$ in distribution. Thus, $L$ is compact in the weak topology. The assertion that \eqref{eqn:equiv1} holds then follows by exactly the same contradiction argument used in the proof of Example~\ref{ex:dominated-family-uniform-levelset}.
\end{remark}

\subsection{Additional material for Section~\ref{sec:application}}\label{app:additional-application}
By considering Gaussian convolutions, Theorem~\ref{thm:transformer-universality} can be generalized to compact sets $K$ which \emph{do not necessarily} satisfy the uniform level set condition; in particular, this generalization allows discrete measures to belong to $K$. To this end, given $\sigma > 0$ and $\mu \in \mathcal{P}(\R^d)$, let $\mu_{\sigma}$ denote the distribution of the random variable $X + \sigma Z,$ where $X \sim \mu$ and $Z \sim \mathcal{N}(0,I_d)$ are independent. That is,
\begin{align*}
    \mu_\sigma\defeq\mathrm{Law}(X + \sigma Z)=\int_{\R^d}\mathcal{N}(x,\sigma^2 I_d)\,\mu(dx)\,.
\end{align*}
For a compact set $K \subset \mathcal{P}_p(\R^d),$ define $$K_{\sigma} = \{\mu_{\sigma}\colon \mu \in K\}.$$
We first prove that the mapping $K \mapsto K_{\sigma}$ preserves $\sfW_p$-compactness.
\begin{lemma}\label{lem:noise_compact}
    If $K \subset \mathcal{P}_p(\R^d)$ is $\sfW_p$-compact and $\sigma > 0,$ then $K_{\sigma}$ is $\sfW_p$-compact. 
\end{lemma}
\begin{proof}
    It suffices to show that the mapping $\mu \mapsto \mu_{\sigma}$ is $\sfW_p$-continuous. In fact, we will show that it is $1$-Lipschitz continuous uniformly in $\sigma$. Fix $\mu\in \mathcal{P}_p(\R^d)$ and $\nu \in \mathcal{P}_p(\R^d)$. Let $(X,Y)$ be random variables such that $X \sim \mu$, $Y \sim \nu$, and $X$ and $Y$ are $\sfW_p$-optimally coupled. If $Z \sim \mathcal{N}(0,I_d)$, then $(X+\sigma Z, Y + \sigma Z)$ defines a coupling of $\mu_{\sigma}$ and $\nu_{\sigma}.$ Therefore,
    \begin{align*}
        \sfW_p(\mu_{\sigma},\nu_{\sigma}) &\leq \bigl(\E\|(X+\sigma Z) - (Y + \sigma Z)\|^p\bigr)^{1/p} \\
        &= \bigl(\E\|X-Y\|^{p}\bigr)^{1/p} \\
        &= \sfW_p(\mu,\nu).
    \end{align*}
    This proves the desired claim.
\end{proof}

Moreover, $\mu_\sigma$ uniformly approximates $\mu$ over $K$ with an explicit rate.
\begin{lemma}\label{lem:noise_rate}
    If $K \subset \mathcal{P}_p(\R^d)$ is $\sfW_p$-compact and $\sigma > 0$, then
    \begin{align*}
        \sup_{\mu \in K} \sfW_p(\mu_\sigma,\mu) \leq \bigl(\E_{Z \sim \mathcal{N}(0,I_d)}\|Z\|^p\bigr)^{1/p}\sigma.
    \end{align*}
\end{lemma}
\begin{proof}
    If $\mu \in K$, $X \sim \mu,$ and $Z \sim \mathcal{N}(0,I_d),$ then $(X,X+\sigma Z)$ defines a coupling of $\mu$ and $\mu_{\sigma}$. Therefore,
    \begin{align*}
        \sup_{\mu \in K} \sfW_p(\mu,\mu_{\sigma}) &\leq \sup_{\mu \in K} \bigl(\E_{X \sim \mu, Z \sim \mathcal{N}(0,I_d)}\|X-(X+\sigma Z)\|^p\bigr)^{1/p} \\
        &= \bigl(\E_{Z \sim \mathcal{N}(0,I_d)}\|Z\|^p\bigr)^{1/p} \sigma.
    \end{align*}
    This completes the proof.
\end{proof}

Notice that $K_{\sigma} \subset \mathcal{P}_{\rm ac}(\R^d)$ for any $\sigma > 0,$ even if $K$ contains singular or atomic measures. Application of Theorem~\ref{thm:transformer-universality} to $K_{\sigma}$, together with the preceding lemmas, yields the following generalization to Theorem~\ref{thm:transformer-universality} that is valid for \emph{any} $\sfW_p$-compact set $K$; in particular, the uniform level set condition \eqref{eq:uniform-level-set-condition} is \emph{not required}.
\begin{corollary}\label{cor:noise_approx}
Fix $p \in [1,\infty)$. Let $F\colon \mathcal{P}_p(\R^d) \rightarrow \mathcal{P}_p(\R^{d'})$ be $\sfW_p$-continuous. Let $K \subset \mathcal{P}_p(\R^d)$ be $\sfW_p$-compact. For any $\varepsilon > 0$, there exists $\sigma > 0$ and a transformer $\mathsf{T}$ of the form \eqref{eq:in-cotext-map} such that
\begin{align*}
    \sup_{\mu \in K} \sfW_p\bigl(\mathsf{T}(\mu_\sigma,\slot)\push\mu_\sigma, F(\mu)\bigr) < \varepsilon.
\end{align*}
\end{corollary}
\begin{proof}
    Let $\varepsilon>0$. Fix $\sigma>0$ to be determined. By Lemma~\ref{lem:noise_compact}, $K_{\sigma}$ is a $\sfW_p$-compact subset of $\mathcal{P}_{\rm ac}(\R^d).$ Thus, by Example~\ref{ex:dominated-family-uniform-levelset}, $K_{\sigma}$ satisfies the uniform level set condition \eqref{eq:uniform-level-set-condition} from Assumption~\ref{ass:uniform-level-set-condition}. Consequently, by Theorem~\ref{thm:transformer-universality}, there exists a deep transformer in-context map $\sfT=\sfT_{\sigma,\ep}\colon \mathcal{P}(\R^d)\times\R^d\to\R^{d'}$ of the form \eqref{eq:in-cotext-map} such that
    \[
    \sup_{\nu\in K_\sigma}
    \sfW_p\bigl(
    \sfT(\nu,\slot)_\#\nu,
    F(\nu)
    \bigr)
    <\frac{\varepsilon}{2}.
    \]
    By Lemma~\ref{lem:noise_rate}, the continuity of $F$, and a result on uniform continuity near compact sets \citep[Lemma~C.2]{tsimpos2026one}, we may choose $\sigma=\sigma(\ep)$ such that
	\begin{align*}
		\sup_{\mu \in K}\sfW_p\bigl(F(\mu_\sigma), F(\mu)\bigr)<\frac{\ep}{2}\,.
	\end{align*}
    Application of the triangle inequality to obtain
    \begin{align*}
        \sup_{\mu \in K} \sfW_p\bigl(\mathsf{T}(\mu_\sigma,\slot)\push\mu_\sigma, F(\mu)\bigr)\leq \sup_{\nu\in K_\sigma}
    \sfW_p\bigl(
    \sfT(\nu,\slot)_\#\nu,
    F(\nu)
    \bigr) + \sup_{\mu \in K}\sfW_p\bigl(F(\mu_\sigma), F(\mu)\bigr) <\ep
    \end{align*}
    completes the proof.
\end{proof}

\medskip

We conclude with the following observation.
The proof of Theorem~\ref{thm:transformer-universality}---in particular the analysis leading to \eqref{eqn:indep_interest_ua_bound}---implies an independently interesting result on the universality of deep transformer in-context maps over $\sfW_p$-compact sets $K\subset\mathcal{P}_p(\R^d)$ of probability measures supported on \emph{the whole of} $\R^d$ instead of on a compact token domain. We state this result as a proposition.
\begin{proposition}
\label{prop:transformer-universality-incontext-unbounded}
Fix $p\in[1,\infty)$. Let $\Omega\subset\R^d$ be compact and $K \subset \mathcal{P}_p(\mathbb R^d)$ be $\sfW_p$-compact.
Then for any $G\in \mathcal{C}(K\times\Omega,\R^{d'})$ and any $\varepsilon>0$, there exists a deep transformer in-context map $\sfT\colon \mathcal{P}(\mathbb R^d)\times\mathbb R^d \to \mathbb R^{d'}$ of the form \eqref{eq:in-cotext-map} such that
    \begin{align*}
        \sup_{(\mu,x)\in K\times\Omega}
        \abs{\sfT(\mu,x)- G(\mu,x)}
        <\varepsilon.
    \end{align*}
\end{proposition}

\section{Proofs of results in Section~\ref{sec:universal-approx}}
This appendix provides the remaining proofs of results from Section~\ref{sec:universal-approx} in the main text.

\subsection{Proof of Example~\ref{ex:one-dimensional-uniform-levelset}}
\label{app:one-dimensional-uniform-levelset}

\begin{proof}
Since $\varphi(x)=x$, we have
\[
\mu\bigl(
\{x\in\mathbb R\colon |\varphi(x)-t|\le r\}
\bigr)
=
\mu\bigl([t- r,t+ r]\bigr).
\]
Thus, it is enough to prove that
\[
\lim_{ r\downarrow0}
\sup_{\mu\in K}\sup_{t\in\mathbb R}
\mu\bigl([t- r,t+ r]\bigr)
=0.
\]
Suppose by contradiction that this fails.
Then there exist
\[
\varepsilon>0,\qquad
 r_n\downarrow0,\qquad
\mu_n\in K,\qquad
t_n\in\mathbb R
\]
such that
\[
\mu_n\bigl([t_n- r_n,t_n+ r_n]\bigr)
\ge\varepsilon
\qfa n\ge1.
\]
Since $K$ is compact, it is tight.
Hence there exists $R>0$ such that
\[
\mu\bigl([-R,R]\bigr)>1-\frac{\varepsilon}{2}
\qfa \mu\in K.
\]
We claim that $(t_n)_{n\in\N}$ is bounded.
Indeed, if $(t_n)_{n\in\N}$ were unbounded, then, after passing to a subsequence,
\[
|t_n|\to\infty.
\]
Since $ r_n\to0$, for all sufficiently large $n$, it holds that
\[
[t_n- r_n,t_n+ r_n]
\subset
\mathbb R\setminus[-R,R].
\]
Therefore
\[
\mu_n\bigl([t_n- r_n,t_n+ r_n]\bigr)
\le
\mu_n\bigl(\mathbb R\setminus[-R,R]\bigr)
<
\frac{\varepsilon}{2},
\]
which contradicts the choice of $\mu_n$ and $t_n$.
Thus $(t_n)_{n\in\N}$ is bounded. Passing to a subsequence, we may assume that
\[
t_n\to t_*
\]
for some $t_*\in\mathbb R$.
Since $K$ is compact, after passing to a further subsequence,
\[
\mu_n\to\mu_*
\]
for some $\mu_*\in K$. Since $\mu_*$ is non-atomic,
\[
\mu_*\bigl(\{t_*\}\bigr)=0.
\]
Therefore, by continuity from above,
\[
\mu_*\bigl([t_*-\delta,t_*+\delta]\bigr)
\to
\mu_*\bigl(\{t_*\}\bigr)=0
\qas \delta\downarrow0.
\]
Hence we may choose $\delta>0$ such that
\[
\mu_*\bigl([t_*-\delta,t_*+\delta]\bigr)
<
\frac{\varepsilon}{2}.
\]
For all sufficiently large $n$, it holds that
\[
[t_n- r_n,t_n+ r_n]
\subset
[t_*-\delta,t_*+\delta].
\]
Consequently,
\[
\varepsilon
\le
\mu_n\bigl([t_*-\delta,t_*+\delta]\bigr).
\]
Since $[t_*-\delta,t_*+\delta]$ is closed, the Portmanteau theorem gives
\[
\limsup_{n\to\infty}
\mu_n\bigl([t_*-\delta,t_*+\delta]\bigr)
\le
\mu_*\bigl([t_*-\delta,t_*+\delta]\bigr).
\]
It follows that
\[
\varepsilon
\le
\mu_*\bigl([t_*-\delta,t_*+\delta]\bigr)
<
\frac{\varepsilon}{2},
\]
a contradiction.
Therefore,
\[
\lim_{ r\downarrow0}
\sup_{\mu\in K}\sup_{t\in\mathbb R}
\mu\bigl([t- r,t+ r]\bigr)
=0,
\]
which proves the assertion.
\end{proof}

\subsection{Proof of Example~\ref{ex:dominated-family-uniform-levelset}}
\label{app:dominated-family-uniform-levelset}

\begin{proof}
Since $\rho$ is non-atomic, by  Lemma~\ref{lem:generic-nonatomic-projection} there exists
$v\in\mathbb S^{d-1}$ such that
\[
x\mapsto \varphi(x)\defeq x\cdot v
\]
satisfies
\[
\varphi_\#\rho\in\mathcal P_{\rm na}(\mathbb R).
\]
Fix $\mu\in K$.
Since $\mu\ll\rho$, for every Borel set $B\subset\mathbb R$, it holds that
\[
(\varphi_\#\rho)(B)=0\quad\text{implies}\quad  \rho\bigl(\varphi^{-1}(B)\bigr)=0.
\]
Hence
\[
\mu\bigl(\varphi^{-1}(B)\bigr)=0,
\]
and therefore
\[
(\varphi_\# \mu)(B)=0.
\]
Thus
\[
\varphi_\# \mu\ll\varphi_\# \rho.
\]
Since $\varphi_\#\rho$ is non-atomic, $\varphi_\#\mu$ is also non-atomic.
Define
\[
L\defeq \{\varphi_\#\mu\colon \mu\in K\}
\subset \mathcal P_p(\mathbb R).
\]
Since $\varphi$ is Lipschitz, the map
\[
\mu\mapsto\varphi_\#\mu
\]
is continuous. 
Therefore $L$ is compact in $\mathcal P_p(\mathbb R)$.
Moreover,
\[
L\subset\mathcal P_{\rm na}(\mathbb R).
\]
We claim that
\[
\lim_{ r\downarrow0}
\sup_{\nu\in L}\sup_{t\in\mathbb R}
\nu\bigl([t- r,t+ r]\bigr)
=0.
\]
Suppose by contradiction that this fails.
Then there exist
$\varepsilon>0$,
$ r_n\downarrow0$,
$\nu_n\in L$,
and $t_n\in\mathbb R$
such that
\[
\nu_n\bigl([t_n- r_n,t_n+ r_n]\bigr)
\ge \varepsilon
\qfa n\ge1.
\]
Since $L$ is compact, it is tight.
Hence there exists $R>0$ such that
\[
\nu\bigl([-R,R]\bigr)>1-\frac{\varepsilon}{2}
\qfa \nu\in L.
\]
Consequently, $(t_n)_{n\in\N}$ must be bounded; otherwise,
for infinitely many $n$, it would hold that
\[
[t_n- r_n,t_n+ r_n]
\subset \mathbb R\setminus[-R,R],
\]
which would imply
\[
\nu_n\bigl([t_n- r_n,t_n+ r_n]\bigr)
\le \frac{\varepsilon}{2}.
\]
This is a contradiction. Now passing to a subsequence,
\[
t_n\to t_*
\]
for some $t_*\in\mathbb R$.
Since $L$ is compact, after passing to a further
subsequence,
\[
\nu_n \to \nu_*
\]
for some $\nu_*\in L$.
Because $\nu_*$ is non-atomic,
\[
\nu_*\bigl(\{t_*\}\bigr)=0.
\]
Hence there exists $\delta>0$ such that
\[
\nu_*\bigl([t_*-\delta,t_*+\delta]\bigr)
<\frac{\varepsilon}{2}.
\]
For sufficiently large $n$,
\[
[t_n- r_n,t_n+ r_n]
\subset
[t_*-\delta,t_*+\delta].
\]
Therefore
\[
\varepsilon
\le
\nu_n\bigl([t_*-\delta,t_*+\delta]\bigr).
\]
Since $[t_*-\delta,t_*+\delta]$ is closed, the Portmanteau theorem yields
\[
\limsup_{n\to\infty}
\nu_n\bigl([t_*-\delta,t_*+\delta]\bigr)
\le
\nu_*\bigl([t_*-\delta,t_*+\delta]\bigr).
\]
Hence
\[
\varepsilon
\le
\nu_*\bigl([t_*-\delta,t_*+\delta]\bigr),
\]
contradicting the choice of $\delta$.
Thus
\[
\lim_{ r\downarrow0}
\sup_{\mu\in K}\sup_{t\in\mathbb R}
\mu\bigl(
\{x\in\mathbb R^d\colon |\varphi(x)-t|\le r\}
\bigr)
=
\lim_{ r\downarrow0}
\sup_{\nu\in L}\sup_{t\in\mathbb R}
\nu\bigl([t- r,t+ r]\bigr)
=0
\]
as asserted.
\end{proof}

\medskip

In the preceding proof, we required the following lemma concerning the non-atomicity of almost every one-dimensional projection.
\begin{lemma}
\label{lem:generic-nonatomic-projection}
Let $\rho\in\mathcal P_{\rm na}(\mathbb R^d)$. 
Then for surface-almost every $v\in\mathbb S^{d-1}$, it holds that
\[
(\pi_v)_\#\rho\in\mathcal P_{\rm na}(\mathbb R),\qw
\pi_v(x)\defeq x\cdot v.
\]
\end{lemma}
\begin{proof}
The case $d=1$ is immediate because $\mathbb S^0=\{-1,1\}$ and $\pi_v(x)=\pm x$.
We therefore assume that $d\ge2$.
Let $\sigma$ denote the surface measure on $\mathbb S^{d-1}$. 
For $v\in\mathbb S^{d-1}$, set
\[
\nu_v\defeq (\pi_v)_\#\rho.
\]
Let
\[
\Delta\defeq \{(s,s)\colon s\in\mathbb R\}\subset\mathbb R^2.
\]
By the definition of pushforward measures,
\[
(\nu_v\otimes\nu_v)(\Delta)
=
(\rho\otimes\rho)
\bigl(
\{(x,y)\in\mathbb R^d\times\mathbb R^d\colon 
v\cdot(x-y)=0\}
\bigr).
\]
Integrating with respect to $v$ and applying Tonelli's theorem gives
\begin{align*}
\int_{\mathbb S^{d-1}}
(\nu_v\otimes\nu_v)(\Delta)\,\sigma(dv)
&=
\int_{\mathbb R^d\times\mathbb R^d}
\sigma\bigl(
\{v\in\mathbb S^{d-1}\colon v\cdot(x-y)=0\}
\bigr)
\,(\rho\otimes\rho)(dx,dy).
\end{align*}
If $x\neq y$, then
\[
\{v\in\mathbb S^{d-1}\colon v\cdot(x-y)=0\}
\]
is a $(d-2)$-dimensional great subsphere of $\mathbb S^{d-1}$ and therefore has zero surface measure. 
If $x=y$, the above set is the
whole sphere. Hence
\[
\int_{\mathbb S^{d-1}}
(\nu_v\otimes\nu_v)(\Delta)\,\sigma(dv)
=
\sigma(\mathbb S^{d-1})\,
(\rho\otimes\rho)
\bigl(
\{(x,x)\colon x\in\mathbb R^d\}
\bigr).
\]
Since $\rho$ is non-atomic,
\[
(\rho\otimes\rho)
\bigl(
\{(x,x)\colon x\in\mathbb R^d\}
\bigr)
=
\int_{\mathbb R^d}\rho(\{x\})\,\rho(dx)
=0.
\]
Therefore,
\[
\int_{\mathbb S^{d-1}}
(\nu_v\otimes\nu_v)(\Delta)\,\sigma(dv)=0.
\]
Since the integrand is nonnegative,
\[
(\nu_v\otimes\nu_v)(\Delta)=0
\]
for $\sigma$ almost every $v\in\mathbb S^{d-1}$.
Last, for any Borel probability measure $\nu$ on $\mathbb R$, it holds that
\[
(\nu\otimes\nu)(\Delta)
=
\int_{\mathbb R}\nu(\{t\})\,\nu(dt).
\]
It follows that $(\nu\otimes\nu)(\Delta)=0$ if and only if $\nu$ is non-atomic because if $\nu$ has an atom $t_0$ of mass $w>0$, then 
\begin{align*}
    (\nu\otimes\nu)(\Delta)\geq (\nu\otimes\nu)\bigl(\{(t_0,t_0)\}\bigr)=(\nu\otimes\nu)\bigl(\{t_0\}\times\{t_0\}\bigr)=\nu(\{t_0\})^2=w^2>0.
\end{align*}
Consequently,
\[
\nu_v=(\pi_v)_\#\rho\in\mathcal P_{\rm na}(\mathbb R)
\]
for $\sigma$ almost every $v\in\mathbb S^{d-1}$.
\end{proof}

\subsection{Proof of Example~\ref{ex:curve-measure-uniform}}
\label{app:curve-measure-uniform}
\begin{proof}
Define
\[
E\defeq 
\Bigl\{
\bigl(\gamma,\gamma(s)\bigr)\colon
\gamma\in K',\ s\in[0,1]
\Bigr\}
\subset K'\times Q.
\]
Consider the map
\[
\Psi\colon K'\times[0,1]\to E,
\qquad
\Psi(\gamma,s)\defeq \bigl(\gamma, \gamma(s)\bigr).
\]
The map $\Psi$ is continuous. Indeed, convergence in $\mathcal{C}^1([0,1],Q)$ implies uniform convergence, so the evaluation map $(\gamma,s)\mapsto\gamma(s)$ is continuous.
Moreover, $\Psi$ is injective. In fact, if
\[
\bigl(\gamma, \gamma(s)\bigr)=\bigl(\widetilde\gamma, \widetilde\gamma(t)\bigr),
\]
then $\gamma=\widetilde\gamma$, and hence
$\gamma(s)=\gamma(t)$. Since $\gamma$ is injective, we obtain $s=t$.
Since $K'\times[0,1]$ is compact and $ K'\times Q$ is Hausdorff,
$\Psi$ is a homeomorphism from $K'\times[0,1]$ onto $E$.
In particular, $E$ is compact.
Define
\begin{align*}
f\colon E\to[0,1], \qw f\bigl(\gamma,\gamma(s)\bigr)\defeq s.
\end{align*}
Equivalently,
\[
f=\pi_2\circ\Psi^{-1},
\]
where $\pi_2(\gamma,s)=s$. Hence $f$ is continuous.

Since $K'\times Q$ is a compact metric space, it is normal.
By the Tietze extension theorem, there exists
\[
\widetilde\varphi\in \mathcal{C}(K'\times Q ,[0,1])
\]
such that
\[
\widetilde\varphi|_E=f.
\]
Therefore,
\[
\widetilde\varphi\bigl(\gamma,\gamma(s)\bigr)=s
\qfa \gamma\in K'\qa s\in[0,1].
\]

Since $K'$ is compact in $\mathcal{C}^1([0,1],Q)$, there exists $C<\infty$ such that
\[
|\dot\gamma(s)|\le C
\qquad
\qfa \gamma\in K'\qa s\in[0,1].
\]
On the other hand, by assumption,
\[
\int_0^1|\dot\gamma(s)|\,ds\ge c
\qfa \gamma\in K'.
\]
Fix $\gamma\in K'$ and $t\in\mathbb R$.
By the definition of $\sigma_\gamma$ and the identity $\widetilde\varphi(\gamma,\gamma(s))=s$, we have
\begin{align*}
&\sigma_\gamma
\bigl
\{
x\in Q\colon
|\widetilde\varphi(\gamma,x)-t|\le r
\}
\bigr)
=
\frac{
\displaystyle
\int_0^1
\mathbf 1_{\{|s-t|\le r\}}
|\dot\gamma(s)|\,ds
}{
\displaystyle
\int_0^1|\dot\gamma(s)|\,ds
}.
\end{align*}
Hence
\[
\sigma_\gamma
\bigl(
\{
x\in Q\colon
|\widetilde\varphi(\gamma,x)-t|\le r
\}
\bigr)
\le
\frac{C}{c}
\mathcal L^1
\bigl([0,1]\cap[t- r,t+ r]\bigr)
\le
\frac{2C}{c} r.
\]
The estimate is uniform in $\gamma\in K'$ and $t\in\mathbb R$.
Therefore,
\[
\lim_{ r\downarrow0}
\sup_{\gamma\in K'}
\sup_{t\in\mathbb R}
\sigma_\gamma
\bigl(
\{
x\in Q\colon
|\widetilde\varphi(\gamma,x)-t|\le r
\}
\bigr)
=0.
\]
It remains to relate this parametrized statement to the uniform level set condition on a family of measures.
First, the map
\[
K'\ni\gamma\mapsto\sigma_\gamma\in\mathcal P_p(Q)
\]
is continuous. 
Indeed, if $\gamma_n\to\gamma$ in $\mathcal{C}^1$, then
\[
\gamma_n\to\gamma
\quad\text{and}\quad
|\dot\gamma_n|\to|\dot\gamma|
\]
uniformly on $[0,1]$, and hence, for every
$h\in \mathcal{C}(Q)$, it holds that
\[
\int_Q h\,d\sigma_{\gamma_n}
=
\frac{
\displaystyle
\int_0^1
h(\gamma_n(s))|\dot\gamma_n(s)|\,ds
}{
\displaystyle
\int_0^1|\dot\gamma_n(s)|\,ds
}
\to
\frac{
\displaystyle
\int_0^1
h(\gamma(s))|\dot\gamma(s)|\,ds
}{
\displaystyle
\int_0^1|\dot\gamma(s)|\,ds
}
=
\int_Q h\,d\sigma_\gamma.
\]
Thus, $\sigma_{\gamma_n}\to\sigma_\gamma$ weakly.
Since $Q$ is compact, this is equivalent to convergence in
$\sfW_p$.
Consequently, 
\[
K\defeq \{\sigma_\gamma\colon \gamma\in K'\}
\]
is compact in $\mathcal P_p(Q)$. 

Assume now that distinct curves in $K'$ have distinct images, i.e., 
\[
\gamma_1([0,1])=\gamma_2([0,1])
\quad\text{implies}\quad
\gamma_1=\gamma_2.
\]
Then, $\gamma\mapsto\sigma_\gamma$ is injective.
Since it is a continuous bijection from the compact space $K'$
onto the Hausdorff space $K$, its inverse
\[
K\ni\sigma_\gamma\mapsto\gamma\in K'
\]
is continuous. We may therefore define
\[
\varphi\colon K\times Q\to[0,1],
\qw
\varphi(\sigma_\gamma, x)
\defeq 
\widetilde\varphi(\gamma, x).
\]
This map is continuous, and the preceding estimate gives
\[
\sup_{\mu\in K}\sup_{t\in\mathbb R}
\mu
\bigl(
\{
x\in Q\colon
|\varphi(\mu,x)-t|\le r
\}
\bigr)
\le
\frac{2C}{c} r.
\]
Hence, $K$ satisfies the uniform level set condition \eqref{eq:uniform-level-set-condition}.
\end{proof}

\subsection{Proof of Theorem~\ref{thm:main-multi}}
\label{app:proof-main-theorem}
\begin{proof}
The proof proceeds in six steps.

\medskip
\noindent\textbf{Step 1 (Truncate the output space).}
For $R>0$, define the truncation map
\[
\Pi_{R}\colon\mathbb R^{d'}\to\overline{B_R(0)}\subset \mathbb R^{d'}\quad\text{by}\quad
\Pi_{R}(y)
\defeq 
\begin{cases}
y, & |y|\le R,\\[1mm]
R\dfrac{y}{|y|}, & |y|>R,
\end{cases}
\]
and set
\[
F_R(\mu)\defeq (\Pi_{R})_\#F(\mu) \qf \mu\in K.
\]
Then a direct calculation shows that
\[
\sfW_p^p\bigl(F(\mu),F_R(\mu)\bigr)
\le
\int_{\{|y|>R\}}|y|^p\,F(\mu)(dy).
\]
Hence
\[
\sup_{\mu\in K}
\sfW_p\bigl(F(\mu),F_R(\mu)\bigr)
\le
E_{{\rm tail},p}(R),
\]
where
\[
E_{{\rm tail},p}(R)
\defeq 
\left(\sup_{\nu\in F(K)}
\int_{\{|y|>R\}}|y|^p\,\nu(dy)\right)^{1/p}.
\]
Since $F(K)$ is compact in $(\mathcal P_p(\mathbb R^{d'}), \sfW_p)$ by the continuity of $F$, Lemma~\ref{lem:uniform-truncation} shows that
\[
E_{{\rm tail},p}(R)\to0
\qas R\to\infty.
\]

\medskip
\noindent\textbf{Step 2 (Cover the truncated output space).}
Fix $h>0$.
Since $\overline{B_R(0)}\subset\mathbb R^{d'}$ is compact, we can choose an $h$-net
\[
\{y_1,\ldots,y_N\}\subset\overline{B_R(0)}
\]
and a continuous partition of unity $\{\psi_i\}_{i=1}^N\subset \mathcal{C}(\mathbb R^{d'})$ such that
\[
0\le\psi_i\le1,
\qquad
\sum_{i=1}^N\psi_i\equiv1
\qon\overline{B_R(0)},
\]
and
\[
\operatorname{supp}(\psi_i)\cap\overline{B_R(0)}
\subset B_h(y_i)
\]
for each $i$. Moreover, the covering may be chosen so that
\[
N=N(R,h)
\le C_{d'}\left(1+\frac{R}{h}\right)^{d'}
\]
by a standard volumetric bound.
Next, define
\[
\alpha_i(\mu)
\defeq
\int_{\mathbb R^{d'}}\psi_i(y)\,F_R(\mu)(dy)\qa
F_{R,h}(\mu)
\defeq
\sum_{i=1}^N\alpha_i(\mu)\,\delta_{y_i}.
\]
Since $F_R$ is continuous by the $1$-Lipschitz property of $\Pi_R$ and each $\psi_i$ is bounded and continuous,
$\alpha_i\colon K\to[0,1]$ is continuous.
For each $\mu\in K$, define a probability measure $\pi_\mu$ on
$\mathbb R^{d'}\times\mathbb R^{d'}$ by
\[
\pi_\mu
\defeq
\sum_{i=1}^N
\bigl(\mathrm{id},y_i\bigr)_\#
\bigl(\psi_i F_R(\mu)\bigr)
=
\sum_{i=1}^N
\bigl(\psi_i F_R(\mu)\bigr)\otimes\delta_{y_i},
\]
where $\psi_i F_R(\mu)$ denotes the (sub-probability) measure with density $\psi_i$ with respect to $F_R(\mu)$.
Because $F_R(\mu)$ is supported on $\overline{B_R(0)}$ and $\sum_{i=1}^N\psi_i=1$ on $\overline{B_R(0)}$, it holds that $\pi_\mu$ is a coupling of $F_R(\mu)$ and $F_{R,h}(\mu)$.
Moreover, for each $i$, the support condition implies that if $\psi_i(y)>0$, then $|y-y_i|<h$ for $F_R(\mu)$-almost every $y$.
Hence
\begin{align*}
\sfW_p^p\bigl(F_R(\mu),F_{R,h}(\mu)\bigr)
&\le
\int_{\mathbb R^{d'}\times\mathbb R^{d'}}
|y-z|^p\,\pi_\mu(dy,dz)\\
&=
\sum_{i=1}^N
\int_{\mathbb R^{d'}}
|y-y_i|^p\psi_i(y)\,F_R(\mu)(dy)\\
&\le
h^p\sum_{i=1}^N
\int_{\mathbb R^{d'}}
\psi_i(y)\,F_R(\mu)(dy)
=h^p.
\end{align*}
Therefore,
\[
\sup_{\mu\in K}
\sfW_p\bigl(F_R(\mu),F_{R,h}(\mu)\bigr)
\le h\qfa R>0.
\]

\medskip
\noindent\textbf{Step 3 (Transport to and from the uniform distribution).}
Since $K\subset\mathcal{P}_p(\R^d)$ satisfies the uniform level set condition \eqref{eq:uniform-level-set-condition} from Assumption~\ref{ass:uniform-level-set-condition} by hypothesis, the construction in Lemma~\ref{lem:Uniform-approximate-uniformizer} delivers the existence of $H_0\in \mathcal{C}(K\times \mathbb{R}^d,[0,1]) $ such that
\begin{equation}
\label{eq:quantitative-uniformizer-main-proof-1}
H_0(\mu,\slot)_\#\mu=
\mathrm{Unif}[0,1]\qfa\mu\in K.
\end{equation}
Next, define the cumulative weights
\[
s_0(\mu)\defeq 0,
\qa
s_i(\mu)\defeq \sum_{\ell=1}^i\alpha_\ell(\mu)
\qfa i=1,\ldots,N.
\]
Then $s_N(\mu)=1$.
Also define the discontinuous map
\[
(\mu,u)\mapsto Q_0(\mu,u)
\defeq 
\sum_{i=1}^N
\mathbf 1_{(s_{i-1}(\mu),s_i(\mu)]}(u) \, y_i.
\]
If $U\sim\mathrm{Unif}[0,1]$, then a direct calculation shows that
\[
\operatorname{Law}\bigl(Q_0(\mu,U)\bigr)=Q_0(\mu,\slot)\push \mathrm{Unif}[0,1]
=
F_{R,h}(\mu).
\]
In particular, \eqref{eq:quantitative-uniformizer-main-proof-1} shows that
\begin{align*}
    F_{R,h}(\mu) = Q_0\bigl(\mu,H_0(\mu,\slot)\bigr)\push\mu.
\end{align*}

\medskip

\noindent\textbf{Step 4 (Regularize the discontinuous pushforward map).}
Let $\delta\in(0,1)$. Let $\kappa\in \mathcal{C}_c^\infty((-1,1))$ be nonnegative with $\int_{\mathbb R}\kappa(r)\,dr=1$. Set
\[
\kappa_\delta(r)\defeq \delta^{-1}\kappa(r/\delta).
\]
Extend $Q_0(\mu,\slot)$ to $\mathbb R$ by setting
\[
\overline Q_0(\mu,u)
\defeq 
\begin{cases}
y_1, & u\leq 0,\\
Q_0(\mu, u), & 0<u\leq1,\\
y_N, & u>1.
\end{cases}
\]
Define
\[
Q_\delta(\mu,u)
\defeq 
\int_{\mathbb R}
\kappa_\delta(u-v)\overline Q_0(\mu, v)\,dv.
\]
Equivalently,
\[
Q_\delta(\mu,u)
=
\sum_{i=1}^N w_i^\delta(\mu, u)\, y_i,
\]
where the nonnegative continuous weights $w_i^\delta$ are obtained by integrating $\kappa_\delta(u-\slot)$ over the intervals associated with $y_i$. 
In particular,
\[
\sum_{i=1}^Nw_i^\delta(\mu,u) = 1
\]
for all $\mu\in K$ and $u\in[0,1]$.

Last, define our final approximation $G$ to be
\[
(\mu,x)\mapsto G(\mu,x)\defeq G_{R,\delta,h}(\mu,x)
\defeq 
Q_\delta\bigl(\mu,H_0(\mu,x)\bigr).
\]
Since $H_0$ and $Q_\delta$ are continuous, it holds that
\[
G_{R,\delta,h}
\in
\mathcal{C}(K\times \mathbb{R}^d,\mathbb R^{d'})
\]
is also continuous.

\medskip

\noindent\textbf{Step 5 (Bound the regularized pushforward error).}
We next compare $Q_\delta$ with $Q_0$. Let
\[
T_\delta(\mu)
\defeq
\bigcup_{i=1}^{N-1}
\{u\in[0,1]\colon |u-s_i(\mu)|<\delta\}
\]
for every $\mu\in K$. Then
\[
|T_\delta(\mu)|\le 2N(R,h)\delta,
\]
and $Q_\delta(\mu,\slot)=Q_0(\mu,\slot)$ outside $T_\delta(\mu)$, apart from a set of Lebesgue measure zero.
Since both maps take values in $\overline{B_R(0)}$, it holds that
\[
|Q_\delta(\mu,u)-Q_0(\mu,u)|\le 2R
\]
for each $\mu$ and $u$. Coupling them by the same random variable
$U\sim\mathrm{Unif}[0,1]$ gives
\begin{align*}
\sfW_p^p\bigl(
G_{R,\delta,h}(\mu,\slot)\push\mu,
F_{R,h}(\mu)
\bigr)&=
\sfW_p^p\bigl(
Q_\delta(\mu,\slot)_\#\mathrm{Unif}[0,1],
Q_0(\mu,\slot)\push \mathrm{Unif}[0,1]
\bigr)\\
&=
\sfW_p^p\bigl(
\operatorname{Law}(Q_\delta(\mu,U)),
\operatorname{Law}(Q_0(\mu,U))
\bigr)
\\
&\le
\mathbb E|Q_\delta(\mu, U)-Q_0(\mu, U)|^p
\\
&\le
(2R)^p|T_\delta(\mu)|
\\
&\le
(2R)^p\,2N(R,h)\delta.
\end{align*}
Therefore,
\begin{equation*}
\label{eq:error-smoothing-main}
\sup_{\mu\in K}
\sfW_p\bigl(
G_{R,\delta,h}(\mu,\slot)\push\mu,
F_{R,h}(\mu)
\bigr)
\le
2R\bigl(2N(R,h)\delta\bigr)^{1/p}.
\end{equation*}

\medskip
\noindent\textbf{Step 6 (Combine the estimates).}
Combining Steps 1--5 and using the triangle inequality, we obtain
\begin{align}\label{eq:quantitative-main-bound}
\sup_{\mu\in K}
\sfW_p\bigl(
G_{R,\delta,h}(\mu,\slot)_\#\mu,
F(\mu)
\bigr)
\le
2R\bigl(2N(R,h)\delta\bigr)^{1/p}
+h
+E_{{\rm tail},p}(R).
\end{align}
Choose $R>0$ sufficiently large such that
\[
E_{{\rm tail},p}(R)<\frac{\varepsilon}{3}.
\]
Next, choose $h>0$ sufficiently small such that
\[
h<\frac{\varepsilon}{3}.
\]
With $R$ and $h$ fixed as in the preceding displays, choose $\delta\in(0,1)$ sufficiently small such that
\[
2R\bigl(2N(R,h)\delta\bigr)^{1/p}
<\frac{\varepsilon}{3}.
\]
Then \eqref{eq:quantitative-main-bound} yields
\[
\sup_{\mu\in K}
\sfW_p\bigl(
G_{R,\delta,h}(\mu,\slot)_\#\mu,
F(\mu)
\bigr)
<\varepsilon.
\]
This completes the proof.
\end{proof}

\medskip

The following result is a well-known property of compact sets in Wasserstein space that is used in the preceding argument; we provide a proof for the sake of completeness.
\begin{lemma}[Uniform integrability of $p$-th moments]
\label{lem:uniform-truncation}
Let $d\in\mathbb N$ and $1\le p<\infty$. Let $K\subset\mathcal P_p(\mathbb R^d)$ be $\sfW_p$-compact.
Then
\begin{equation}
\label{eq:uniform-tail}
\lim_{R\to\infty}
\sup_{\nu\in K}
\int_{\{|y|>R\}}|y|^p\,\nu(dy)
=0.
\end{equation}
\end{lemma}

\begin{proof}
Suppose by contradiction that \eqref{eq:uniform-tail} fails.
Then there exist $\delta>0$, a sequence $(\nu_n)_{n\in\mathbb N}\subset K$, and a sequence $R_n\to\infty$ such that
\[
\int_{\{|y|>R_n\}}|y|^p\,\nu_n(dy)\ge\delta
\qfa n\in\mathbb N.
\]
Since $K$ is compact with respect to $\sfW_p$, after passing to a subsequence, we may assume that
\[
\nu_n\to\nu
\qin \sfW_p
\]
for some $\nu\in K$.

Fix $\eta>0$.
Since $\nu\in\mathcal P_p(\mathbb R^d)$, there exists $M>0$ such that
\[
\int_{\{|y|>M\}}|y|^p\,\nu(dy)<\eta.
\]
Choose a continuous function $\chi_M\colon [0,\infty)\to[0,1]$ such that
\[
\chi_M(r)=0 \qfa r\le M\qa
\chi_M(r)=1 \qfa r\ge M+1,
\]
and define
\[
\varphi_M(y)\defeq |y|^p\chi_M(|y|) \qfa  y\in\mathbb R^d.
\]
Then $\varphi_M$ is continuous and satisfies
\[
0\le\varphi_M(y)\le |y|^p,
\]
as well as
\[
\varphi_M(y)=0 \qif |y|\le M\qa
\varphi_M(y)=|y|^p\qif |y|\ge M+1.
\]

Since $\nu_n\to\nu$ in $\sfW_p$ and $\varphi_M$ is continuous with at most $p$-th-order growth,
\[
\int_{\mathbb R^d}\varphi_M(y)\,\nu_n(dy)
\to
\int_{\mathbb R^d}\varphi_M(y)\,\nu(dy)\qas n\to\infty.
\]
Moreover,
\[
\int_{\mathbb R^d}\varphi_M(y)\,\nu(dy)
\le
\int_{\{|y|>M\}}|y|^p\,\nu(dy)
<\eta.
\]
Hence, for all sufficiently large $n$, it holds that
\[
\int_{\mathbb R^d}\varphi_M(y)\,\nu_n(dy)<2\eta.
\]
Since $R_n\to\infty$, we also have $R_n\ge M+1$ for all sufficiently large $n$. 
Therefore,
\begin{align*}
\int_{\{|y|>R_n\}}|y|^p\,\nu_n(dy)
&\le
\int_{\{|y|>M+1\}}|y|^p\,\nu_n(dy)\\
&\le
\int_{\mathbb R^d}\varphi_M(y)\,\nu_n(dy)\\
&<2\eta.
\end{align*}
Choosing $\eta<\delta/2$ gives a contradiction.
\end{proof}

\medskip

The core of the proof of Theorem~\ref{thm:main-multi} relies on transporting to the one-dimensional uniform distribution with a jointly continuous in-context map.
\begin{lemma}[Continuous exact uniformizer]
\label{lem:Uniform-approximate-uniformizer}
Let $p\in[1,\infty)$. Let $K\subset\mathcal{P}_p(\mathbb{R}^d)$ satisfy Assumption~\ref{ass:uniform-level-set-condition}. Then there exists
\[
H_0\in \mathcal{C}(K\times\mathbb{R}^d,[0,1])
\]
such that
\[
H_0(\mu,\slot)_\#\mu=
\mathrm{Unif}[0,1]\qfa \mu\in K.
\]
\end{lemma}

\begin{proof}
Let $\varphi$ be the continuous scalarization from Assumption~\ref{ass:uniform-level-set-condition}. Fix $ r>0$. Define
\begin{align}\label{eqn:level_set_residual}
\omega_K( r)
\defeq 
\sup_{\mu\in K}\sup_{t\in\mathbb R}
\mu\bigl(
\{x\in\mathbb{R}^d\colon
|\varphi(\mu,x)-t|\le r\}
\bigr).
\end{align}
Let $\Theta\colon \mathbb R\to[0,1]$ be continuous and nondecreasing such that
\[
\Theta(s)=0
\qfa s\le-1\qa
\Theta(s)=1
\qfa s\ge1.
\]
Define
\[
H_ r(\mu,x)
\defeq 
\int_{\mathbb{R}^d}
\Theta\left(
\frac{
\varphi(\mu,x)-\varphi(\mu,z)
}{r}
\right)
\,\mu(dz).
\]
Clearly $H_ r$ takes values in $[0,1]$.

We first show that
\[
H_ r\in \mathcal{C}(K\times \mathbb{R}^d, [0,1]).
\]
To this end, let $(\mu_n,x_n)\to(\mu,x)$ in $K\times \mathbb{R}^d$ as $n\to\infty$ and set
\[
g_n(z)
\defeq 
\Theta\left(
\frac{
\varphi(\mu_n,x_n)-\varphi(\mu_n,z)
}{r}
\right)\qa
g(z)
\defeq 
\Theta\left(
\frac{
\varphi(\mu,x)-\varphi(\mu,z)
}{r}
\right).
\]
Since $\varphi$ is continuous, $g_n\to g$ uniformly on every compact subset of $\mathbb{R}^d$. 
Moreover,
\[
0\le g_n\le1\qa 0\le g\le1
\]
and each are continuous. Since $\mu_n\to\mu$ in $\sfW_p$, in particular $\mu_n$ converges to $\mu$ in distribution, and the family $\{\mu_n\colon n\in\mathbb N\}\cup\{\mu\}$ is uniformly tight.
It follows that
\[
\int_{\mathbb{R}^d} 
g_n(z)\,\mu_n(dz)
\to
\int_{\mathbb{R}^d} g(z)\,\mu(dz)
\]
as $n\to\infty$ by the same argument used in Remark~\ref{rem:uniform-level-set-condition-equiv}. Hence
\[
H_ r(\mu_n,x_n)\to H_ r(\mu,x)\qas n\to\infty,
\]
so $H_ r$ is continuous.

Now fix $\mu\in K$ and define
\[
\nu_\mu
\defeq 
\varphi(\mu,\slot)_\#\mu
\in\mathcal P(\mathbb R).
\]
Assumption~\ref{ass:uniform-level-set-condition} implies that $\nu_\mu$ is non-atomic. To see this, note that for every $t\in\mathbb R$, it holds that
\[
\nu_\mu(\{t\})
=
\mu\bigl(
\{x\in\mathbb{R}^d\colon\varphi(\mu,x)=t\}
\bigr)
\]
and
\[
\nu_\mu(\{t\})
\le
\mu\bigl(
\{x\in\mathbb{R}^d\colon 
|\varphi(\mu,x)-t|\le\eta
\}
\bigr)
\]
for every $\eta>0$. Letting $\eta\downarrow0$ and using the uniform level set condition
gives
\begin{align}\label{eqn:pw_atomless}
    \nu_\mu(\{t\})=0
\end{align}
as claimed.

Continuing, we let
\[
F_\mu(t)
\defeq 
\nu_\mu\bigl((-\infty,t]\bigr)
\qf t\in\mathbb R
\]
be the cumulative distribution function of $\nu_\mu$. Define
its smoothed version by
\[
F_{\mu, r}(t)
\defeq 
\int_{\mathbb R}
\Theta\left(
\frac{t-s}{ r}
\right)
\,\nu_\mu(ds).
\]
Then
\[
H_ r(\mu,x)
=
F_{\mu, r}\bigl(\varphi(\mu,x)\bigr),
\]
and consequently
\[
H_ r(\mu,\slot)_\#\mu
=
(F_{\mu, r})_\#\nu_\mu.
\]
For every $t\in\mathbb R$, it holds that
\[
|F_{\mu, r}(t)-F_\mu(t)|
\le
\nu_\mu([t- r,t+ r]).
\]
Indeed, the functions
\[
s\mapsto
\Theta\left(\frac{t-s}{ r}\right)
\quad\text{and}\quad
s\mapsto\mathbf 1_{(-\infty,t]}(s)
\]
coincide outside the set $[t- r,t+ r]$.
Hence,
\begin{align}\label{eqn:cdf_sup_bound}
\sup_{t\in\mathbb R}
|F_{\mu, r}(t)-F_\mu(t)|
\le \omega_K( r)
\qfa \mu\in K.
\end{align}
Since $\nu_\mu$ is non-atomic, $F_\mu$ is continuous. The
probability integral transform gives
\[
(F_\mu)_\#\nu_\mu
=
\mathrm{Unif}[0,1].
\]
Last, define
\begin{align*}
    H_0\colon K\times \R^d&\to [0,1],\\
    (\mu,x)&\mapsto F_\mu\bigl(\varphi(\mu,x)\bigr).
\end{align*}
By the preceding two displays, it holds that
\begin{align*}
    H_0(\mu,\slot)\push\mu= (F_\mu)\push\nu_\mu = \mathrm{Unif}[0,1]
\end{align*}
for all $\mu\in K$. It remains to show that $H_0$ is jointly continuous. To this end, let $(\mu_n,x_n)\to (\mu,x)$ as $n\to\infty$ in $K\times \R^d$. For any $r>0$, we estimate
\begin{align*}
    \abs{H_0(\mu_n,x_n)-H_0(\mu,x)}&\leq     
    \abs{H_0(\mu_n,x_n)-H_r(\mu_n,x_n)}\\
&\qquad\qquad +     \abs{H_r(\mu_n,x_n)-H_r(\mu,x)}\\
 & \qquad\qquad\qquad\qquad +     \abs{H_r(\mu,x)-H_0(\mu,x)}\\
 &\leq 2\omega_K(r) + \abs{H_r(\mu_n,x_n)-H_r(\mu,x)}
\end{align*}
by \eqref{eqn:level_set_residual} and \eqref{eqn:cdf_sup_bound}. The second term in the last line of the preceding display tends to zero as $n\to\infty$ for fixed $r$ by continuity of $H_r$. After sending $n\to\infty$, the first term $2\omega_K(r)$ tends to zero as $r\to 0$ by Assumption~\ref{ass:uniform-level-set-condition}. Thus, $H_0$ is continuous as asserted.
\end{proof}

\section{Proofs of results in Section~\ref{sec:application}}
This appendix provides the remaining proofs of results from Section~\ref{sec:application} in the main text.

\subsection{Proof of Theorem~\ref{thm:transformer-universality}}
\label{app:transformer-universality}

\begin{proof}
Fix $\varepsilon>0$.
By Theorem~\ref{thm:main-multi} and Remark~\ref{rem:quantitative-main}, there exists a \emph{bounded} continuous map
\[
G\colon K\times \mathbb R^d\to\mathbb R^{d'}
\]
such that
\begin{equation}
\label{eq:cor1-G-F}
\sup_{\mu\in K}
\sfW_p\bigl(G(\mu,\slot)_\#\mu,F(\mu)\bigr)
<\frac{\varepsilon}{2}.
\end{equation}
Set
\[
M_G\defeq
\sup_{(\mu,x)\in K\times \R^d}|G(\mu,x)|<\infty.
\]

We first compactify the input space.
Let
\[
\Omega\defeq[-1,1]^d
\]
and define
\begin{align*}
    \iota\colon \mathbb R^d&\to(-1,1)^d,\\
    x&\mapsto \iota(x)\defeq\frac{x}{1+|x|}.
\end{align*}
The map $\iota$ is bounded, continuous, and injective. 
Since it is $1$-Lipschitz, the induced pushforward
\begin{align*}
    \iota_\#\colon \mathcal P_p(\mathbb R^d)&\to\mathcal P(\Omega),\\
\mu&\mapsto\iota_\#\mu
\end{align*}
is continuous with respect to $\sfW_p$. 
Define
\[
K^\iota\defeq
\{\iota_\#\mu\in\mathcal{P}(\Omega)\colon \mu\in K\}\subset\mathcal P_p(\mathbb{R}^d).
\]
Since $K$ is compact, $K^\iota$ is also compact by continuity of $\iota\push$. 
Moreover,
\[
\iota_\#|_K\colon K\to K^\iota
\]
is a homeomorphism.

Since $K$ is compact in $(\mathcal P_p(\mathbb R^d),\sfW_p)$, it is uniformly tight. 
Hence
\[
\tau_R\defeq
\sup_{\mu\in K}\mu\bigl(\mathbb R^d\setminus B_R(0)\bigr)
\to0
\qas R\to\infty.
\]
Choose $R>0$ sufficiently large such that
\begin{equation}
\label{eq:cor1-R-choice}
2M_G\tau_R^{1/p}<\frac{\varepsilon}{4}.
\end{equation}
Set
\[
C_R\defeq\iota\bigl(\overline{B_R(0)}\bigr)\subset\Omega.
\]
Then $C_R$ is compact. 
Define
\[
\widetilde G_R\colon
K^\iota\times C_R\to\overline{B_{M_G}(0)}
\]
by
\[
\widetilde G_R\bigl(\iota_\#\mu,\iota(x)\bigr)
\defeq G(\mu,x)\qfa
\mu\in K\qa x\in\overline{B_R(0)}.
\]
This map is well-defined and continuous because $\iota$ and $\iota_\#|_K$ are homeomorphisms onto their images.

Since $K^\iota\times C_R$ is a closed subset of the compact metric space $\mathcal P(\Omega)\times\Omega$ and $\overline{B_{M_G}(0)}$ is convex, the Dugundji extension theorem \citep{dugundji1951extension} delivers the existence of a continuous map
\[
\widehat G_R\colon
\mathcal P(\Omega)\times\Omega
\to\overline{B_{M_G}(0)}
\]
such that
\begin{equation}
\label{eq:cor1-Ghat-agrees}
\widehat G_R\bigl(\iota_\#\mu,\iota(x)\bigr)
=G(\mu,x)
\qfa 
\mu\in K\qa x\in\overline{B_R(0)},
\end{equation}
and
\begin{equation}
\label{eq:cor1-Ghat-bound}
\sup_{(\nu,z)\in \mathcal P(\Omega)\times\Omega}|\widehat G_R(\nu,z)|\le M_G.
\end{equation}

Choose $\delta\in(0,1)$ sufficiently small such that
\begin{equation}
\label{eq:cor1-delta-choice}
\delta<\frac{\varepsilon}{8}.
\end{equation}
By the universal approximation theorem for measure-theoretic transformer in-context maps on the compact domain $\mathcal{P}(\Omega)\times\Omega$ \citep[Theorem~1]{furuya2024transformers}, there exists a measure-theoretic transformer
\[
\sfT_\delta\colon 
\mathcal P(\Omega)\times\Omega\to\mathbb R^{d'}
\]
of the form \eqref{eq:in-cotext-map}---also depending on $R$---such that
\begin{equation}
\label{eq:cor1-transformer-approx}
\sup_{(\nu,z)\in\mathcal P(\Omega)\times\Omega}
\bigl|\sfT_{\delta}(\nu,z)
-\widehat G_R(\nu,z)\bigr|<\delta.
\end{equation}
In particular, by \eqref{eq:cor1-Ghat-bound} and the preceding display, it holds that
\begin{equation}
\label{eq:cor1-transformer-bound}
\sup_{(\nu,z)\in\mathcal P(\Omega)\times\Omega}|\sfT_\delta(\nu,z)|
< \delta + M_G.
\end{equation}

Since $\Omega$ is compact, $\mathcal P(\Omega)\times\Omega$ is compact. Since $\sfT_\delta$ is continuous on $\mathcal P(\Omega)\times\Omega$ \citep{furuya2026approximation}, it is uniformly
continuous. 
Therefore, there exists a nondecreasing function
\[
\omega_\delta\colon [0,\infty)\to[0,\infty)
\]
such that $\lim_{r\downarrow0}\omega_\delta(r)=0$ and
\begin{equation}
\label{eq:cor1-modulus}
\bigl|\sfT_\delta(\nu,z)
-\sfT_\delta(\nu',z')\bigr|
\le
\omega_\delta\bigl(
\sfW_p(\nu,\nu')+|z-z'|
\bigr)
\end{equation}
for every $\nu$ and $\nu'$ in $\mathcal P(\Omega)$ and $z$ and $z'$ in $\Omega$.
Now choose $\rho>0$ such that
\begin{equation}
\label{eq:cor1-rho-choice}
\omega_\delta(\rho)<\frac{\varepsilon}{8}.
\end{equation}
After $\delta$ and $\rho$ have been fixed, choose $S\ge R$ sufficiently large such that
\begin{equation}
\label{eq:cor1-S-choice}
\tau_S\defeq
\sup_{\mu\in K}\mu\bigl(\mathbb R^d\setminus B_S(0)\bigr)
<
\left(\frac{\rho}{4\sqrt d}\right)^p.
\end{equation}

We next approximate $\iota$ on $\overline{B_S(0)}$ by a token-wise MLP.
Define
\[
\operatorname{clip}(t)
\defeq
-1+\operatorname{ReLU}(t+1)-\operatorname{ReLU}(t-1)
\]
and
\[
\operatorname{Clip}(z)
\defeq
\bigl(\operatorname{clip}(z_1),\ldots,
\operatorname{clip}(z_d)\bigr).
\]
Then $\operatorname{Clip}$ is exactly representable by a shallow ReLU neural network and satisfies
\[
\operatorname{Clip}(\mathbb R^d)\subset\Omega.
\]
For any $\eta>0$, there exists a ReLU MLP $\widetilde N_\eta\colon\mathbb R^d\to\mathbb R^d$ satisfying
\[
\sup_{x\in\overline{B_S(0)}}
|\widetilde N_\eta(x)-\iota(x)|<\eta
\]
by the universal approximation theorem for ReLU MLPs. Set
\[
N_\eta\defeq\operatorname{Clip}\circ\widetilde N_\eta,
\]
which is also a MLP. Since $\operatorname{Clip}$ is the metric projection onto $\Omega=[-1,1]^d$, it is $1$-Lipschitz. Thus,
\begin{equation}
\label{eq:cor1-N-approx}
\sup_{x\in\overline{B_S(0)}}
|N_\eta(x)-\iota(x)|=\sup_{x\in\overline{B_S(0)}}
\bigl|\operatorname{Clip}\bigl(\widetilde{N}_\eta(x)\bigr)-\operatorname{Clip}\bigl(\iota(x)\bigr)\bigr|<\eta.
\end{equation}
Moreover,
\[
N_\eta(\mathbb R^d)\subset\Omega.
\]
Choose $\eta>0$ sufficiently small such that
\begin{equation}
\label{eq:cor1-eta-choice}
\eta<\frac{\rho}{4}.
\end{equation}

Let $\Gamma_{\rm id}$ be an attention layer of the form \eqref{eq-mf-self-attn} with $Q^h=K^h=V^h=0$ for every
head $h$. 
Then
\[
\Gamma_{\rm id}(\mu,x)=x
\]
for every $(\mu,x)\in
\mathcal P(\mathbb R^d)\times\mathbb R^d$.
Define $\sfT\colon \mathcal P(\mathbb R^d)\times\mathbb R^d\to\R^{d'}$ by
\begin{equation*}
\sfT
\defeq
\sfT_{\delta}\diamond N_\eta
\diamond\Gamma_{\rm id}.
\end{equation*}
Equivalently,
\[
\sfT(\mu,x)
=
\sfT_{\delta}
\bigl((N_\eta)\push\mu, N_\eta(x)\bigr).
\]
The initial identity attention layer and the token-wise MLP $N_\eta$ put $\sfT$ in the form \eqref{eq:in-cotext-map}. 
The initial layer $\Gamma_{\rm id}$ is defined for
every probability measure. Then $\mu\mapsto (N_{\eta})\push\mu$ maps $\mathcal{P}(\R^d)$ into $\mathcal{P}([-1,1]^d)$. Consequently, all subsequent attention layers act on measures supported on compact subsets of Euclidean space because measure-theoretic attention operators map compactly supported probability measures to compactly supported probability measures.
Thus, $\sfT$ is well-defined on $\mathcal P(\mathbb R^d)\times\mathbb R^d$.

We next estimate the error incurred by replacing $\iota$ with $N_\eta$.
Let $\mu\in K$. The probability measure $(N_\eta,\iota)_\#\mu$ is a coupling of $(N_\eta)_\#\mu$ and $\iota_\#\mu$. 
Hence
\begin{align*}
\sfW_p^p\bigl((N_\eta)_\#\mu,\iota_\#\mu\bigr)
&\le
\int_{\mathbb R^d}
|N_\eta(x)-\iota(x)|^p\,\mu(dx)\\
&\le
\eta^p+(2\sqrt d)^p
\mu\bigl(\mathbb R^d\setminus B_S(0)\bigr)
\end{align*}
by splitting the integral and invoking \eqref{eq:cor1-N-approx}.
Using $(\abs{a}^p+\abs{b}^p)^{1/p}\le \abs{a}+\abs{b}$ yields
\begin{equation*}
\sup_{\mu\in K}
\sfW_p\bigl((N_\eta)_\#\mu,\iota_\#\mu\bigr)
\le
\eta+2\sqrt d\,\tau_S^{1/p},
\end{equation*}
where $\tau_S$ is as in \eqref{eq:cor1-S-choice}.
So, for $x\in B_R(0)\subset B_S(0)$, it holds that
\begin{align*}
\sfW_p\bigl((N_\eta)_\#\mu,\iota_\#\mu\bigr)
+|N_\eta(x)-\iota(x)|
\le
2\eta+2\sqrt d\,\tau_S^{1/p}
<
\frac{\rho}{2}+\frac{\rho}{2}
=\rho
\end{align*}
by \eqref{eq:cor1-eta-choice} and \eqref{eq:cor1-S-choice}. Then by \eqref{eq:cor1-modulus} and
\eqref{eq:cor1-rho-choice}, it holds that
\begin{equation*}
\bigl|
\sfT(\mu,x)
-
\sfT_{\delta}\bigl(\iota_\#\mu,\iota(x)\bigr)\bigr|=
\bigl|
\sfT_{\delta}
\bigl((N_\eta)_\#\mu,N_\eta(x)\bigr)
-
\sfT_{\delta}\bigl(\iota_\#\mu,\iota(x)\bigr)\bigr|
\leq \omega_\delta(\rho)
<
\frac{\varepsilon}{8}.
\end{equation*}
Moreover, \eqref{eq:cor1-Ghat-agrees} and \eqref{eq:cor1-transformer-approx} yield
\begin{align*}
\sup_{\mu\in K}\bigl|
\sfT_{\delta}\bigl(\iota_\#\mu,\iota(x)\bigr)
-G(\mu,x)
\bigr|<\delta.
\end{align*}
Thus, if $x\in B_R(0)$, then
\begin{align}\label{eqn:indep_interest_ua_bound}
    \abs{\sfT(\mu,x)-G(\mu,x)}<\frac{\ep}{8} + \delta.
\end{align}
In the other hand, if $x\in\mathbb R^d\setminus B_R(0)$, then \eqref{eq:cor1-transformer-bound} and a crude bound show that
\begin{equation*}
\label{eq:cor1-embedding-outside}
\abs{\sfT(\mu,x)-G(\mu,x)}\leq (\delta + M_G) + M_G=\delta+2M_G.
\end{equation*}
Consequently,
\begin{align}\label{eq:cor1-integrated-embedding}
\abs{\sfT(\mu,x)-G(\mu,x)}\leq  \frac{\ep}{8}\mathbf{1}_{B_R(0)}(x) + \delta + 2M_G\mathbf{1}_{\R^d\setminus B_R(0)}(x)\qfa x\in\R^d.
\end{align}

Using the common random variable $X\sim\mu$ to couple $\sfT(\mu,X)$ and $G(\mu,X)$, followed by Minkowski's inequality and \eqref{eq:cor1-integrated-embedding}, we obtain the estimate
\begin{equation}
\label{eq:cor1-transformer-G-approx}
\begin{aligned}
\sup_{\mu\in K}
\sfW_p\bigl(
\sfT(\mu,\slot)_\#\mu,
G(\mu,\slot)_\#\mu
\bigr)
\le
\frac{\ep}{8}
+\delta+ 2M_G\tau_R^{1/p}<\frac{\ep}{2}
\end{aligned}
\end{equation}
by \eqref{eq:cor1-delta-choice} and \eqref{eq:cor1-R-choice}. 
Last, the triangle inequality, \eqref{eq:cor1-transformer-G-approx}, and \eqref{eq:cor1-G-F} deliver the bound
\[
\begin{aligned}
\sup_{\mu\in K}
\sfW_p\bigl(
\sfT(\mu,\slot)_\#\mu,
F(\mu)
\bigr)
&\le
\sup_{\mu\in K}
\sfW_p\bigl(
\sfT(\mu,\slot)_\#\mu,
G(\mu,\slot)_\#\mu
\bigr)\\
&\qquad\qquad+
\sup_{\mu\in K}
\sfW_p\bigl(
G(\mu,\slot)_\#\mu, F(\mu)
\bigr)\\
&<
\frac{\varepsilon}{2}
+\frac{\varepsilon}{2}
=\varepsilon
\end{aligned}
\]
as asserted.
\end{proof}

\subsection{Proof of Corollary~\ref{cor:main-cross}}\label{app:main-cross-proof}
\begin{proof}
The proof follows the same construction as that of
Theorem~\ref{thm:main-multi}, but now with the input measure $\mu$ replaced
by the continuously-varying source measure $\eta(\mu)$.
Indeed, the assumption \eqref{eq:cross-uniform-level-set}, together with
the continuity of $\eta$, yields, by the same argument as in
Lemma~\ref{lem:Uniform-approximate-uniformizer}, a continuous map
\[
\widetilde{H}_0\in \mathcal{C}\bigl(K\times \mathbb R^m,[0,1]\bigr)
\]
such that
\[
\widetilde{H}_0(\mu,\slot)_\#\eta(\mu)=
\operatorname{Unif}[0,1]
\qfa \mu\in K.
\]
The remainder of the proof is identical to that of
Theorem~\ref{thm:main-multi}, except now with
$H_0(\mu,\slot)_\#\mu$ replaced by
$\widetilde{H}_0(\mu,\slot)_\#\eta(\mu)$.
\end{proof}

\end{document}